\documentclass{article}

\PassOptionsToPackage{numbers, compress}{natbib}

\usepackage[final]{neurips2021}

\usepackage[utf8]{inputenc} 
\usepackage[T1]{fontenc}    
\usepackage{hyperref}       
\usepackage{url}            
\usepackage{booktabs}       
\usepackage{amsfonts}       
\usepackage{nicefrac}       
\usepackage{microtype}      
\usepackage{xcolor}         
\usepackage{tcolorbox}
\usepackage{graphicx}

\usepackage{multirow}

\usepackage{bm}
\usepackage{amsmath,amssymb,amsfonts,amsthm}
\usepackage{mathtools}
\usepackage[inline]{enumitem}

\usepackage{caption}
\usepackage{subcaption}

\usepackage{tikz}
\usetikzlibrary{bayesnet,calc}

\makeatletter
\newtheorem*{rep@lemma}{\rep@title}
\newcommand{\newreplemma}[2]{%
\newenvironment{rep#1}[1]{%
 \def\rep@title{#2 \ref{##1}}%
 \begin{rep@lemma}}%
 {\end{rep@lemma}}}

\newtheorem*{rep@theorem}{\rep@title}
\newcommand{\newreptheorem}[2]{%
\newenvironment{rep#1}[1]{%
 \def\rep@title{#2 \ref{##1}}%
 \begin{rep@theorem}}%
 {\end{rep@theorem}}}

\newtheorem*{rep@corollary}{\rep@title}
\newcommand{\newrepcorollary}[2]{%
\newenvironment{rep#1}[1]{%
 \def\rep@title{#2 \ref{##1}}%
 \begin{rep@corollary}}%
 {\end{rep@corollary}}}

\newtheorem*{rep@proposition}{\rep@title}
\newcommand{\newrepproposition}[2]{%
\newenvironment{rep#1}[1]{%
 \def\rep@title{#2 \ref{##1}}%
 \begin{rep@proposition}}%
 {\end{rep@proposition}}}
\makeatother

\newtheorem{lemma}{Lemma}[section]
\newtheorem{theorem}{Theorem}[section]
\newtheorem{corollary}{Corollary}[section]

\theoremstyle{remark}

\theoremstyle{remark}

\theoremstyle{definition}
\newtheorem{definition}{Definition}[section]

\theoremstyle{plain}

\newtheorem{proposition}{Proposition}

\newreplemma{lemma}{Lemma}
\newreptheorem{theorem}{Theorem}
\newrepcorollary{corollary}{Corollary}
\newrepproposition{proposition}{Proposition}

\DeclareMathOperator\argmin{arg\,min}

\newcommand\CPr[2]{\ensuremath{\mathbb{P}\!\left(#1\,\middle\vert\,#2\right)}}

\newcommand\E[1]{\ensuremath{\mathbb{E}\!\left[#1\right]}}
\newcommand\EE[2]{\ensuremath{\mathbb{E}_{#1}\!\left[#2\right]}}
\newcommand\CE[2]{\ensuremath{\mathbb{E}\!\left[#1\,\middle\vert\,#2\right]}}
\newcommand\CEE[3]{\ensuremath{\mathbb{E}_{#1}\!\left[#2\,\middle\vert\,#3\right]}}

\newcommand\V[1]{\ensuremath{\mathbb{V}\!\left[#1\right]}}

\newcommand\norm[1]{\ensuremath{\left\lvert\left\lvert#1\right\rvert\right\rvert}}

\newcommand\bth{\ensuremath{\bm{\theta}}}

\def\gradthi_#1{\ensuremath{\nabla_{\theta_{#1}}}}

\def\gradbth{\ensuremath{\nabla_{\bth}}}

\def\p(#1){\ensuremath{p\!\left(#1\right)}}
\def\cp(#1|#2){\ensuremath{p\!\left(#1\,\middle\vert\, #2\right)}}

\def\pp_#1(#2){\ensuremath{p_{#1}\!\left(#2\right)}}
\def\cpp_#1(#2|#3){\ensuremath{p_{#1}\!\left(#2\,\middle\vert\, #3\right)}}

\def\pt_#1(#2){\ensuremath{\pi_{#1}\!\left(#2\right)}}
\def\cpt_#1(#2|#3){\ensuremath{\pi_{#1}\!\left(#2\,\middle\vert\, #3\right)}}

\def\dot<#1, #2>{\ensuremath{\left\langle#1, #2\right\rangle}}

\renewcommand\Re{\ensuremath{\mathbb{R}}}

\newcommand{\indep}{\rotatebox[origin=c]{90}{$\models$}}

\title{Factored Policy Gradients: Leveraging Structure\\for Efficient Learning in MOMDPs}

\author{%
  Thomas Spooner\\
  J.\ P.\ Morgan AI Research\\
  \texttt{thomas.spooner@jpmorgan.com} \\
  \And
  Nelson Vadori\\
  J.\ P.\ Morgan AI Research\\
  \texttt{nelson.vadori@jpmorgan.com} \\
  \And
  Sumitra Ganesh\\
  J.\ P.\ Morgan AI Research\\
  \texttt{sumitra.ganesh@jpmorgan.com}
}

\begin{document}
\maketitle

\begin{abstract}
Policy gradient methods can solve complex tasks but often fail when the dimensionality of the action-space or objective multiplicity grow very large. This occurs, in part, because the variance on score-based gradient estimators scales quadratically. In this paper, we address this problem through a factor baseline which exploits independence structure encoded in a novel action-target influence network. Factored policy gradients (FPGs), which follow, provide a common framework for analysing key state-of-the-art algorithms, are shown to generalise traditional policy gradients, and yield a principled way of incorporating prior knowledge of a problem domain's generative processes. We provide an analysis of the proposed estimator and identify the conditions under which variance is reduced. The algorithmic aspects of FPGs are discussed, including optimal policy factorisation, as characterised by minimum biclique coverings, and the implications for the bias-variance trade-off of incorrectly specifying the network. Finally, we demonstrate the performance advantages of our algorithm on large-scale bandit and traffic intersection problems,  providing a novel contribution to the latter in the form of a spatial approximation.
\end{abstract}

\section{Introduction}
Many sequential decision-making problems in the real-world have objectives that
can be naturally decomposed into a set of conditionally independent targets.
Control of water reservoirs, energy consumption optimisation, market making,
cloud computing allocation, sewage flow systems, and robotics are but a few
examples~\cite{roijers:2013:survey}. While many optimisation methods have been
proposed~\cite{mannor:2004:geometric,prashanth:2016:variance} --- perhaps most
prominently using Lagrangian scalarisation~\cite{tessler:2018:reward} ---
multi-agent learning has emerged as a promising new paradigm for sample-efficient
learning~\cite{dusparic:2012:autonomic}.
In this class of algorithms, the multi-objective learning problem is cast into
a centralised, co-operative stochastic game in which co-ordination is achieved
through global coupling terms in each agent's objective/reward functions. For
example, a grocer who must manage their stock could be decomposed into a
collection of sub-agents that each manage a single type of produce, but are
subject to a global constraint on inventory. This approach has been shown to be
very effective in a number of
domains~\cite{lee:2002:multi,van:2014:novel,patel:2018:optimizing,mannion:2018:reward,wang:2019:multi},
but presents both conceptual and technical issues.

The transformation of a multi-objective Markov decision process
(MOMDP)~\cite{roijers:2013:survey}
into a stochastic game is a non-trivial design
challenge. In many cases there is no clear delineation between agents in the
new system, nor an established way of performing the decomposition.  What's
more, it's unclear in many domains that a multi-agent perspective is
appropriate, even as a technical trick. For example, the concurrent problems
studied by \citet{silver:2013:concurrent} exhibit great levels of homogeneity,
lending themselves to the use of a shared policy which conditions on contextual
information. The key challenge that we address in this paper is precisely how
to scale these single-agent methods --- specifically, policy gradients --- in a
\emph{principled} way. As we shall see, this study reveals that existing
methods in both single- and multi-agent multi-objective optimisation can be
formulated as special cases of a wider family of algorithms we entitle
\emph{factored policy gradients}. The contributions of this paper are summarised below:
\begin{enumerate}
    \item We introduce \emph{\textbf{influence networks}} as a framework for
        modelling probabilistic relationships between actions and objectives in an
        MOMDP, and show how they can be combined with \emph{\textbf{policy
        factorisation}} via graph partitioning.
    \item We propose a new control variate --- the \emph{\textbf{factor baseline}} --- that exploits
        independence structures within a (factored) influence network, and show
        how this gives rise to a novel class of algorithms to which we ascribe the name
        \emph{\textbf{factored policy gradients}}.
    \item We show that FPGs \emph{\textbf{generalise traditional policy
        gradient estimators}} and provide a common framework for analysing
        state-of-the-art algorithms in the literature including
        action-dependent baselines and counterfactual policy gradients.
    \item The \emph{\textbf{variance properties}} of our family of algorithms are studied, and \emph{\textbf{minimum factorisation}} is put forward as a principled way of applying FPGs, with theoretical results around the existence and uniqueness of the characterisation.
    \item The final contribution is to illustrate the effectiveness of our
        approach over traditional estimators on two
        \emph{\textbf{high-dimensional benchmark domains}}.
\end{enumerate}

\subsection{Related Work}
\textbf{Policy gradients.}
Variance reduction techniques in the context of policy gradient methods have
been studied for some time. The seminal work of \citet{konda:2000:actor} was
one of the earliest works that identified the use of a critic as beneficial for
learning. Since then, baselines (or, control variates) have received much
attention. In \citeyear{weaver:2001:optimal}, \citet{weaver:2001:optimal}
presented the first formal analysis of their properties, and later
\citet{greensmith:2004:variance} proved several key results around optimality.
More recently, these techniques have been extended to include action-dependent
baselines~\cite{thomas:2017:policy,liu:2018:action,grathwohl:2018:backpropagation,wu:2018:variance,foerster:2018:counterfactual},
though the source of their apparent success has been questioned by
some~\cite{tucker:2018:mirage} who suggest that subtle implementation details
were the true driver. It has also been shown that one can reduce variance by
better accounting for the structure of the action-space, such as
bounds~\cite{chou:2017:improving,fujita:2018:clipped} or more general
topological properties~\cite{eisenach:2019:marginal}. The SVRPG approach of
\citet{papini:2018:stochastic} also addresses variance concerns in policy
gradients by leveraging advances in supervised learning, and the generalised
advantage estimator of \citet{schulman:2015:high} has been proposed as a method
for reducing variance in actor-critic methods with fantastic empirical results;
both of these can be combined with baselines and the techniques we present
in this work.
\textbf{Factorisation.} In a related, but distinct line of work,
factorisation has been proposed to better leverage the transition structure of MDPs; see
e.g.~\cite{boutilier:1995:planning,guestrin:2003:efficient,tamar:2012:integrating}.
Indeed, the notion of causality has also been utilised in work by
\citet{jonsson:2006:causal}. Most recently, \citet{oliehoek:2012:influence}
presented an elegant framework for harnessing the \emph{influence} of other
agents (from the perspective of self) in multi-agent systems. This approach is
complementary to the work presented in this paper, and more recent extensions
have significantly advanced the
state-of-the-art~\cite{suau:2019:influence,congeduti:2020:loss,oliehoek:2021:sufficient};
we build upon these principles. There is also a long line of research on ``influence diagrams'' that is pertinent to this work. While the majority of this effort has been focused on dynamic programming, the ideas are very closely related to ours and indeed we see this work as a natural extension of these concepts~\cite{tatman:1990:dynamic}.
\textbf{Miscellaneous.} Causal/graphical
modelling has seen past applications in reinforcement learning~\cite{ghavamzadeh:2015:bayesian}.
Indeed, our proposed influence network is related to, but distinct from, the
action influence models introduced by \citet{madumal:2020:explainable} for
explainability. There, the intention was to construct policies that can justify
actions with respect to the observation space. Here, the intention was to
exploit independence structure in MOMDPs for scalability and efficiency.

\section{Background}
A regular discrete-time Markov decision process (MDP) is a tuple $\mathcal{M}
\doteq \left(\mathcal{S}, \mathcal{A}, \mathcal{R}, p, p_0\right)$, comprising:
a \emph{state space} $\mathcal{S}$, \emph{action space} $\mathcal{A}$, and set
of \emph{rewards} $\mathcal{R} \subseteq \mathbb{R}$. The dynamics of the MDP
are driven by an \emph{initial state distribution} such that $s_0 \sim \pp_0(\cdot)$ and
a stationary \emph{transition kernel} where
$\left(r_t,s_{t+1}\right) \sim \cp(\cdot, \cdot|s_t,\bm{a}_t)$ satisfies the
Markov property, $\cp(r_t,s_{t+1}|h_t) = \cp(r_t,s_{t+1}|s_t,\bm{a}_t)$, for
any history $h_t \doteq \left(s_0, \bm{a}_0, r_0, s_1, \dots, s_t,
\bm{a}_t\right)$.
Given an MDP, a (stochastic) policy, parameterised by $\bm{\theta} \in
\mathbb{R}^n$, is a mapping $\pi_{\bth} : \mathcal{S} \times \mathbb{R}^n \to
\mathcal{P}\!\left(\mathcal{A}\right)$ from states and weights to the set of
probability measures on $\mathcal{A}$. The conditional probability density of
an action $\bm{a}$ is denoted by $\cpt_\bth(\bm{a}|s) \doteq \CPr{\bm{a} \in
\textrm{d} \bm{a}}{s, \bth}$ and we assume throughout that $\pi_{\bm{\theta}}$
is continuously differentiable with respect to $\bm{\theta}$. For a given
policy, the \emph{return} starting from time $t$ is defined as the discounted
sum of future rewards, $G_t \doteq \sum_{k=0}^T \gamma^k r_{t+k+1}$, where
$\gamma \in [0, 1]$ is the discount rate and $T$ is the terminal
time~\cite{sutton:2018:reinforcement}. \emph{Value functions} express the
expected value of returns generated from a given state or state-action pair
under the MDP's transition dynamics and policy $\pi$: that is, $v_\pi(s) \doteq
\CEE{\pi}{G_t}{s_t = s}$ and $q_\pi(s, \bm{a}) \doteq \CEE{\pi}{G_t}{s_t = s,
\bm{a}_t = \bm{a}}$. The objective in \emph{control} is to find a policy that
maximises $v_\pi$ for all states with non-zero measure under $p_0$, denoted by
the Lesbesgue integral $J\!\left(\bth\right) \doteq \EE{p_0}{v_{\pi_{\bth}}(s_0)} =
\int_{\mathcal{S}} v_{\pi_{\bth}}(s_0) \, \textrm{d} p_0\!\left(s_0\right)$.

\subsection{Policy Search}
In this paper, we focus on policy gradient methods which optimise the
parameters $\bth$ directly. This is achieved, in general, by performing
gradient ascent on $J\!\left(\bth\right)$, for which \citet{sutton:2000:policy}
derived
\begin{equation}\label{eq:pgrad}
    \gradbth J\!\left(\bth\right) =
    \EE{\pi_{\bth},\rho_{\pi_{\bth}}}{\left(q_{\pi_{\bth}}\!\left(s,
    \bm{a}\right) - b\!\left(s\right)\right) \bm{z}},
\end{equation}
where $\bm{z} \doteq \gradbth\ln\cpt_{\bth}(\bm{a}|s)$ is the policy's score
vector, $\rho_{\pi_{\bth}}\!\left(s\right) \doteq
\int_\mathcal{S}\sum_{t=0}^\infty \gamma^t \cp(s_t = s|\textrm{d}s_0,
\pi_{\bth})$ denotes the (improper) discounted-ergodic occupancy measure, and
$b\!\left(s\right)$ is a state-dependent baseline (or, control
variate)~\cite{peters:2006:policy}. Here, $\cp(s_t = s|s_0,\pi_{\bth})$ is the
probability of transitioning from $s_0 \to s$ in $t$ steps under $\pi_{\bth}$.
\autoref{eq:pgrad} is convenient for a number of reasons:
\begin{enumerate*}
    \item it is a score-based estimator~\cite{mohamed:2020:monte}; and
    \item it falls under the class of stochastic approximation
        algorithms~\cite{borkar:2009:stochastic}.
\end{enumerate*}
This is important as it means $q_\pi\!\left(s, \bm{a}\right)$ may be replaced
by \emph{any} unbiased quantity, say $\psi : \mathcal{S}\times\mathcal{A}\to\Re$, such that $\EE{\pi, \rho_\pi}{\psi\!\left(s,
\bm{a}\right)} = q_\pi\!\left(s, \bm{a}\right)$, while retaining
convergence guarantees. It also implies that optimisation can be performed
using stochastic gradient estimates, the standard variant of which is defined
below.

\begin{definition}[\textbf{VPGs}]\label{def:vpg}
    The \emph{vanilla policy gradient} estimator for target-baseline pair $\left(\psi, b\right)$ is denoted
    \begin{equation}
        \bm{g}^\textsc{V}\!\left(s, \bm{a}\right) \doteq \left[\psi\!\left(s,
        \bm{a}\right) - b\!\left(s\right)\right] \bm{z},
    \end{equation}
    where $\gradbth J\!\left(\bth\right) = \EE{\pi_{\bth},
    \rho_{\pi_{\bth}}}{\bm{g}^\textsc{V}\!\left(s, \bm{a}\right)}$.
\end{definition}

\subsection{Factored (Action-Space) MDPs}
In this paper, we consider the class of MDPs in which the action-space factors
into a product, $\mathcal{A} \doteq \bigotimes\nolimits_{i=1}^n \mathcal{A}_i =
\mathcal{A}_1 \times\dots\times \mathcal{A}_n$, for some $n$. This is satisfied
trivially when $n=1$ and $\mathcal{A}_1 = \mathcal{A}$, but also holds in many
common settings, such as $\mathcal{A} \doteq \Re^n$, which factorises $n$ times
as $\bigotimes_{i=1}^n \Re$. This is equivalent to requiring that actions,
$\bm{a} \in \mathcal{A}$, admit a ``subscript'' operation; without
necessarily having $\mathcal{A}$ be a vector space. For example, one could have an action-space of the form $\mathcal{A} \doteq \Re\times\mathbb{N}$ such that, for any $\bm{a}\in\mathcal{A}$, $a_1 \in \Re$ and $a_2 \in \mathbb{N}$. To this end, we introduce the notion of partition maps which will feature throughout the paper.

\begin{definition}[\textbf{Partition Map}]\label{def:partition_map}
    Define $\mathcal{X} \doteq \bigotimes\nolimits_{i=1}^n\mathcal{X}_i$ and $J
    \subseteq [n]$ with $\mathcal{X}_J \doteq \bigotimes\nolimits_{j\in
    J}\mathcal{X}_j$ such that a partition map (PM) for a pair
    $\left(\mathcal{X}, J\right)$ is a function $\sigma :
    \mathcal{X} \to \mathcal{X}_J$ with complement $\bar\sigma :
    \mathcal{X} \to \mathcal{X}_{[n]\setminus J}$.
\end{definition}

Partition maps are an extension of the canonical projections of the product
topology, and are equivalent to the scope operator used by
\citet{tian:2020:towards}. For example, if $\left(a_1, a_2, a_3\right) \doteq
\bm{a} \in \mathcal{A} \doteq \Re^3$ denotes a three-dimensional real
action-space, then one possible PM is given by $\sigma\!\left(\bm{a}\right) =
\left(a_1, a_3\right)$ with complement $\bar\sigma\!\left(\bm{a}\right) =
\left(a_2\right)$. Note that there should always exist a unique inverse operation
that recovers the original space; in this case, it would be expressed as
$f\!\left(\left(a_1, a_3\right), \left(a_2\right)\right) = \left(a_1, a_2,
a_3\right)$.

\section{Influence Networks}
Consider an MOMDP with scalarised objective given by
\begin{equation}\label{eq:linear_psi}
    J\!\left(\bth\right) \doteq \EE{p_0}{\psi\!\left(s, \bm{a}\right) \doteq
    \sum_{j=1}^m \lambda_j\psi_j\!\left(s,
    \sigma_j\!\left(\bm{a}\right)\right)},
\end{equation}
where $\lambda_j \in \Re$ for all $1 \leq j \leq m$ and each $\psi_j\!\left(s,
\sigma_j\!\left(\bm{a}\right)\right)$ denotes some target that depends on a
single partition of the action components. Traditional MDPs can be seen as a
special case in which $m=1$, and $\psi = \psi_1 \doteq q_\pi$. The vector $\bm{\psi}\!\left(s, \bm{a}\right)$ comprises the
concatenation of all $m$ targets and each partitioning is dictated by the
\emph{non-empty} maps $\sigma_j\!\left(\bm{a}\right)$, the form of which is
intrinsic to the MOMDP. For convenience, let us denote the collection of
targets comprising $\psi\!\left(s, \bm{a}\right)$ by
\begin{equation}\label{eq:psi}
    \Psi \doteq \left\{\psi_j : \psi\!\left(s,
    \bm{a}\right) = \dot<\bm{\lambda}, \bm{\psi}\!\left(s,
    \bm{a}\right)>\right\}.
\end{equation}

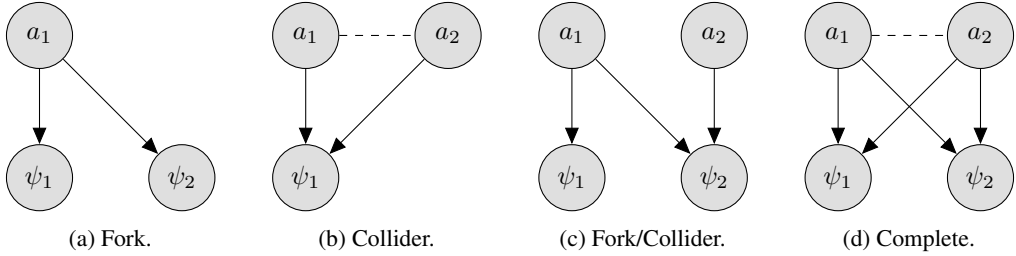
\begin{figure}[t]
    \centering
    \begin{subfigure}[b]{0.24\linewidth}
        \centering
        \begin{tikzpicture}
            \node[obs, minimum size=25pt] (psi1) {$\psi_2$};
            \node[obs, left=of psi1, minimum size=25pt] (psi2) {$\psi_1$};

            \node[obs, above=of psi2, minimum size=25pt] (a1) {$a_{1}$};

            \edge {a1} {psi1};
            \edge {a1} {psi2};


        \end{tikzpicture}

        \caption{Fork.}\label{fig:graph:prototypes:fork}
    \end{subfigure}
    \hfill
    \begin{subfigure}[b]{0.24\linewidth}
        \centering
        \begin{tikzpicture}
            \node[obs, minimum size=25pt] (a2) {$a_{2}$};
            \node[obs, left=of a2, minimum size=25pt] (a1) {$a_{1}$};

            \node[obs, below=of a1, minimum size=25pt] (psi1) {$\psi_1$};

            \edge {a1, a2} {psi1};


            \path (a1) edge[dashed] (a2);
        \end{tikzpicture}

        \caption{Collider.}\label{fig:graph:prototypes:collider}
    \end{subfigure}
    \hfill
    \begin{subfigure}[b]{0.24\linewidth}
        \centering
        \begin{tikzpicture}
            \node[obs, minimum size=25pt] (a2) {$a_{2}$};
            \node[obs, left=of a2, minimum size=25pt] (a1) {$a_{1}$};

            \node[obs, below=of a1, minimum size=25pt] (psi1) {$\psi_1$};
            \node[obs, below=of a2, minimum size=25pt] (psi2) {$\psi_2$};

            \edge {a1} {psi1};
            \edge {a1, a2} {psi2};


        \end{tikzpicture}

        \caption{Fork/Collider.}\label{fig:graph:prototypes:fc}
    \end{subfigure}
    \hfill
    \begin{subfigure}[b]{0.24\linewidth}
        \centering
        \begin{tikzpicture}
            \node[obs, minimum size=25pt] (a2) {$a_{2}$};
            \node[obs, left=of a2, minimum size=25pt] (a1) {$a_{1}$};

            \node[obs, below=of a1, minimum size=25pt] (psi1) {$\psi_1$};
            \node[obs, below=of a2, minimum size=25pt] (psi2) {$\psi_2$};

            \edge {a1, a2} {psi1};
            \edge {a1, a2} {psi2};


            \path (a1) edge[dashed] (a2);
        \end{tikzpicture}

        \caption{Complete.}\label{fig:graph:prototypes:complete}
    \end{subfigure}

    \caption{Influence network prototypes and action-target junction
        patterns~\cite{pearl:2009:causality,pearl:2018:book}. Edges depict
        dependencies between factors $a_i \in \mathcal{A}$ and targets $\psi_j
        \in \Psi$; and dashed lines a partition induced by the minimum
    factorisation.}\label{fig:graph:prototypes}
\end{figure}

The intuition behind FPGs is derived from the observation
that each factor of the action-space only \emph{influences} a subset of the $m$
targets. Take, for example, \autoref{fig:graph:prototypes:fc} which depicts an
instance of an influence network between a 2-dimensional action vector
and a 2-dimensional target. The edges suggest that $a_1$ affects the value of
both $\psi_1$ and $\psi_2$, whereas $a_2$ only affects $\psi_2$. This
corresponds to an objective of the form $\lambda_1\psi_1\!\left(s, a_1\right) +
\lambda_2\psi_2\!\left(s, \bm{a}\right)$, where each goal's domain derives from
the edges of the graph. This is formalised in \autoref{def:influence_net}
below.

\begin{definition}[\textbf{Influence Network}]\label{def:influence_net}
    A bipartite graph $\mathcal{G}\!\left(\mathcal{M}, \Psi\right) \doteq
    (I_\mathcal{A}, I_\Psi, E)$ is said to be the influence network of an MDP
    $\mathcal{M}$ and target set $\Psi$ if for $I_\mathcal{A} \doteq
    \left[\left\lvert\mathcal{A}\right\rvert\right]$ and $I_\Psi \doteq
    \left[\left\lvert\Psi\right\rvert\right]$, the presence of an edge, $e\in
    E$, between nodes $i \in I_\mathcal{A}$ and $j \in I_\Psi$ defines a causal
    relationship between the $i^{\textrm{th}}$ factor of $\mathcal{A}$ and the
    $j^{\textrm{th}}$ target $\psi_j\!\left(s,
    \sigma_j\!\left(\bm{a}\right)\right)$.
\end{definition}

An influence network can be seen as a structural equation
model~\cite{pearl:2009:causality} in which each vertex in $I_\mathcal{A}$ has a
single, unique parent which is exogenous and drives the randomness in action
sampling, and each vertex in $I_\Psi$ has parents only in the set
$I_\mathcal{A}$ as defined by the set of edges $E$. The structural equations along each edge
$(i, j) \in E$ are given by the target functions themselves and the partition
maps $\sigma_j$ mirror the parents of each node $j$. Some examples of influence
networks are illustrated in \autoref{fig:graph:prototypes}; see also the
appendix. We now define the key concept of influence matrices.

\begin{definition}[\textbf{Influence Matrix}]
    Let $\bm{K}_\mathcal{G}$ denote the \emph{biadjacency matrix} of an
    influence network $\mathcal{G}$, defined as the $\lvert I_\mathcal{A}
    \rvert \times \lvert I_\Psi \rvert$ boolean matrix with $K_{ij} = 1 \iff
    (i, j) \in E$ for $i \in I_\mathcal{A}$ and $j \in I_\Psi$.
\end{definition}

Together, these definitions form a calculus for expressing the
relationships between the factors of an action-space and the targets of an
objective of the form in \autoref{eq:linear_psi}. We remark that, from an
algorithmic perspective, we are free to choose between two representations:
graph-based, or partition map-based. The duality between $\mathcal{G}$ and
$\bm{K}$, and the set $\{\sigma_j : j \in I_\Psi\}$, is intrinsic to our choice
of notation and serves as a useful correspondence during analysis.

\subsection{Policy Factorisation}
Influence networks capture the relationships between $\mathcal{A}$ and
$\Psi$, but policies are typically defined over groups of actions rather than
the individual axes of $\mathcal{A}$. Consider, for example, a multi-asset
trading problem in which an agent must quote buy and sell prices for each of
$n$ distinct assets~\cite{gueant:2019:deep,spooner:2020:robust}.  There is a
natural partitioning between each pair of prices and the $n$ sources of
profit/loss, and one might therefore define the policy as a product of $n$
bivariate distributions as opposed to a full joint, or fully factored model.
This choice over \emph{policy factorisation} relates to the independence
assumptions we make on the distribution $\pi_{\bth}$ for the sake of
performance. Indeed, in the majority of the literature, policies are defined
using an isotropic distribution~\cite{wu:2018:variance} since there is no
domain knowledge to motivate more complex covariance structure. We formalise
this below.

\begin{definition}[\textbf{Policy Factorisation}]\label{def:policy_factorisation}
    An $n$-fold \emph{policy factorisation}, $\Sigma \doteq \left\{\sigma^\pi_i : i \in
    \left[n\right]\right\}$, is a set of disjoint partition maps
    that form a complete partitioning over the action space.
\end{definition}

The definition above provides a means of expressing \emph{any} joint policy
distribution in terms of PMs,
\begin{equation}\label{eq:policy_factorisation}
    \cpt_{\bth}(\bm{a}|s) \doteq \prod_{i=1}^n
    \cpt_{i,\bth}(\sigma^\pi_i\!\left(\bm{a}\right)|s),
\end{equation}
where $\sigma_i^\pi\in\Sigma$ and $n=\left\lvert\Sigma\right\rvert$. This
corresponds to a transformation of the underlying influence network where the
action vertices are grouped under the $n$ policy factors and, for any $i, j \in
[n]$, $i\ne j$, we have mutual independence:
$\sigma^\pi_i\!\left(\bm{a}\right)~\indep~\sigma_j^\pi\!\left(\bm{a}\right)$.
This is captured in the following concept.

\begin{definition}[\textbf{Factored Influence Network}]\label{def:fin}
    For a given influence network $\mathcal{G}$ and policy factorisation
    $\Sigma$, we define a \emph{factored influence network},
    $\mathcal{G}_\Sigma$, by replacing $I_\mathcal{A}$ with $I_\Sigma$, the set
    of partitioned vertices, and merge the corresponding edges to give
    $E_\Sigma$. Similarly, denote by $\bm{K}_\Sigma$ the influence matrix with
    respect to the $\Sigma$-factorisation.
\end{definition}

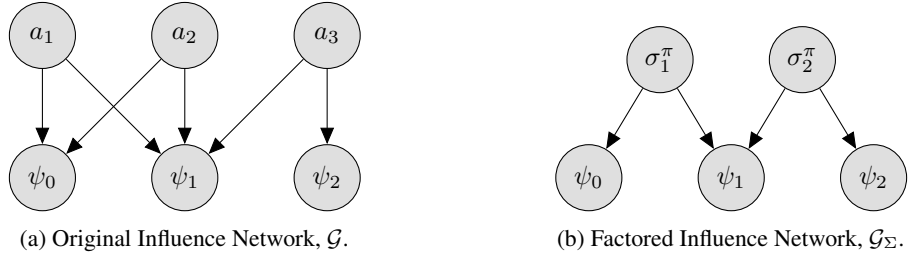
\begin{figure}
    \centering
    \begin{subfigure}[b]{0.48\linewidth}
        \centering
        \begin{tikzpicture}
            \node[obs, minimum size=25pt] (psi1) {$\psi_1$};
            \node[obs, left=of psi1, minimum size=25pt] (psi2) {$\psi_0$};
            \node[obs, right=of psi1, minimum size=25pt] (psi3) {$\psi_2$};

            \node[obs, above=of psi2, minimum size=25pt] (a1) {$a_{1}$};
            \node[obs, above=of psi1, minimum size=25pt] (a2) {$a_{2}$};
            \node[obs, above=of psi3, minimum size=25pt] (a3) {$a_{3}$};

            \edge {a1, a2} {psi1};
            \edge {a1, a2} {psi2};
            \edge {a3} {psi1, psi3};
        \end{tikzpicture}

        \caption{Original Influence Network, $\mathcal{G}$.}
        \label{fig:policy_factorisation:unfactored}
    \end{subfigure}
    \hfill
    \begin{subfigure}[b]{0.48\linewidth}
        \centering
        \begin{tikzpicture}
            \node[obs, minimum size=25pt] (psi1) {$\psi_1$};
            \node[obs, left=of psi1, minimum size=25pt] (psi2) {$\psi_0$};
            \node[obs, right=of psi1, minimum size=25pt] (psi3) {$\psi_2$};

            \node (g1) at ($(psi1)!0.5!(psi2)$) {};
            \node (g2) at ($(psi1)!0.5!(psi3)$) {};

            \node[obs, above=of g1, minimum size=25pt] (p1) {$\sigma^\pi_1$};
            \node[obs, above=of g2, minimum size=25pt] (p2) {$\sigma^\pi_2$};

            \edge {p1} {psi1, psi2};
            \edge {p2} {psi1, psi3};
        \end{tikzpicture}

        \caption{Factored Influence Network, $\mathcal{G}_\Sigma$.}
        \label{fig:policy_factorisation:factored}
    \end{subfigure}

    \caption{Influence network transformation under a $\Sigma$-factorisation
    with $\sigma_1^\pi\!\left(\bm{a}\right) \doteq (a_1, a_2)$ and
    $\sigma_2^\pi\!\left(\bm{a}\right) \doteq (a_3)$. Here, $\Sigma$
    corresponds to a minimum factorisation of the
    policy; i.e.\ $\Sigma = \Sigma^\star$.}\label{fig:policy_factorisation}
\end{figure}

Factored influence networks ascribe links between the policy
factors in \autoref{eq:policy_factorisation} and the targets $\psi_j \in \Psi$.
They play an important role in \autoref{sec:fpgs} and provide a refinement of
\autoref{def:influence_net} which allows us to design more efficient
algorithms. As an example, \autoref{fig:policy_factorisation} shows how one
possible policy factorisation transforms an influence network $\mathcal{G}$
into $\mathcal{G}_\Sigma$. Note that while the action nodes and edges have been
partitioned into policy factors, the fundamental topology with respect to the
attribution of influence remains unchanged; i.e.\ no dependencies are lost.

\section{Factored Policy Gradients}\label{sec:fpgs}
\emph{Factored policy gradients} exploit factored influence networks by
attributing each $\psi_j \in \Psi$ only to the policy factors that were
probabilistically responsible for generating it; that is, those with a connecting
edge in the given $\mathcal{G}_\Sigma$. The intuition is that the extraneous
targets in the objective do not contribute to learning, but do contribute
towards variance. For example, it would be counter-intuitive to include
$\psi_2$ of \autoref{fig:policy_factorisation:factored} in the update for
$\pi_1$ since it played no generative role. Naturally, by removing these terms
from the gradient estimator, we can improve the signal to noise ratio and yield
more stable algorithms. This idea can be formulated into a set of baselines
which are defined and validated below.
\begin{definition}[\textbf{Factor Baselines}]\label{def:factor_baselines}
    For a given $\mathcal{G}_\Sigma$, the \emph{factor baselines} (FBs) are defined as
    \begin{equation}
        b^\textsc{F}_i\!\left(s, \bar\sigma^\pi_i\!\left(\bm{a}\right)\right)
        \doteq \left[\left(\bm{1} -
        \bm{K}_\Sigma\right)\bm{\lambda}\circ\bm{\psi}\!\left(s,
        \bm{a}\right)\right]_i,\label{eq:baseline:factor}
    \end{equation}
    for all $i \in\left[\left\lvert\Sigma\right\rvert\right]$, where $\circ$ denotes the Hadamard product and $\bm{1}$ is to be taken as an all-ones matrix.
\end{definition}

\begin{lemma}\label{lem:factor_baselines}
    FBs are valid control variates if $\mathcal{G}_\Sigma$ is true to the MDP (i.e.\ unbiased).

\end{lemma}

\emph{Factor baselines} are related to the action-dependent baselines studied
by \citet{wu:2018:variance} and \citet{tucker:2018:mirage}, as well as the
methods employed by COMA~\cite{foerster:2018:counterfactual} and
DRPGs~\cite{castellini:2020:difference} in multi-agent systems. Note, however, that FBs are distinct
in two key ways:
\begin{enumerate*}
    \item they adhere to the structure of the influence network and account not
        only for policy factorisation, but also the target multiplicity of
        MOMDPs; and
    \item unlike past work, factor baselines were defined using an ansatz based on the
        structure implied by a given $\mathcal{G}_\Sigma$ as opposed to
        explicitly deriving the $\argmin$ of the variance, or approximation thereof; see
        the appendix.
\end{enumerate*}
This means that, unlike optimal baselines, FBs can be computed efficiently and
thus yield practical algorithms. Indeed, this very fact is why the state-value
function is used so ubiquitously in traditional actor-critic methods as a
state-dependent control variate despite being sub-optimal. It follows that we
can define an analogous family of methods for MOMDPs with zero computational overhead.

\begin{proposition}[\textbf{FPGs}]\label{prop:fpg}
    Take a $\Sigma$-factored policy $\pt_\bth(\bm{a}|s)$ and
    $\left\lvert\bm{\theta}\right\rvert \times \left\lvert\Sigma\right\rvert$
    matrix of scores $\bm{S}\!\left(s, \bm{a}\right)$. Then, for target vector
    $\bm{\psi}\!\left(s, \bm{a}\right)$ and multipliers $\bm{\lambda}$, the
    \emph{FPG} estimator
    \begin{equation}\label{eq:fpg}
        \bm{g}^\textsc{F}\!\left(s, \bm{a}\right) \doteq \bm{S}\!\left(s,
        \bm{a}\right)\bm{K}_\Sigma\,\bm{\lambda}\circ\bm{\psi}\!\left(s,
        \bm{a}\right),
    \end{equation}
    is an unbiased estimator of the true policy gradient; i.e.\ $\gradbth J\!\left(\bth\right) = \EE{\pi_{\bth},
    \rho_{\pi_{\bth}}}{\bm{g}^\textsc{F}\!\left(s, \bm{a}\right)}$.
\end{proposition}

\autoref{prop:fpg} above shows that the VPG estimator given in
\autoref{def:vpg} can be expressed in our calculus as
$\bm{S}\bm{1}\bm{\lambda}\circ\bm{\psi}$, where $\bm{1}$ is an all-ones matrix
and, traditionally, $\psi \doteq q_\pi$; note that one can still include other
baselines in \autoref{eq:fpg} such as $v_\pi$. In other words,
\emph{\autoref{prop:fpg} strictly generalises the policy gradient
theorem}~\cite{sutton:2000:policy} and, by virtue of it's unbiasedness, thus retains all convergence guarantees. We also see that both
COMA~\cite{foerster:2018:counterfactual} and
DRPGs~\cite{castellini:2020:difference} are special cases in which the
influence network reflects the separation of agents with
$\bm{K}_\Sigma$ a square, and often diagonal matrix.

\subsection{Variance Analysis}\label{sec:variance_analysis}
The variance reducing effect of FBs comprises two terms:
\begin{enumerate*}
    \item a quadratic and thus non-negative component which scales with the
        \emph{second moments} of $b_i^\textsc{F}$; and
    \item a linear term which scales with the \emph{expected values} of
        $b_i^\textsc{F}$.
\end{enumerate*}
This is shown in the following result.

\begin{proposition}[Variance Decomposition]\label{prop:variance_reduction}
    Let $\bm{g}_i$ denote a gradient estimate for the $i$\textsuperscript{th}
    factor of a $\Sigma$-factored policy $\pi_{\bth}$
    (\autoref{eq:policy_factorisation}). Then, $\Delta\mathbb{V}_i \doteq
    \V{\bm{g}_i^\textsc{V}} - \V{\bm{g}_i^\textsc{F}}$, satisfies
    \begin{equation}
        \Delta\mathbb{V}_i = \alpha_i
        \,\EE{\bar\sigma_i^\pi\!\left(\bm{a}\right)}{\left(b_i^\textsc{F}\right)^2}
        + 2\beta_i\EE{\bar\sigma_i^\pi\!\left(\bm{a}\right)}{b_i^\textsc{F}},
        \label{eq:var_diff}
    \end{equation}
    where $\bm{z}_i \doteq \gradbth\ln\cpt_{i,\bth}(\bm{a}|s)$, $\alpha_i
    \doteq \EE{\sigma_i^\pi\!\left(\bm{a}\right)}{\dot<\bm{z}_i, \bm{z}_i>}
    \geq 0$ and $\beta_i \doteq
    \EE{\sigma_i^\pi\!\left(\bm{a}\right)}{\dot<\bm{z}_i, \bm{z}_i>\left(\psi +
    b_i^\textsc{F}\right)}$.
\end{proposition}

The first of these two terms is a ``free lunch'' which removes the targets
that are not probabilistically related to each factor. The linear term, on the other
hand, couples the adjusted target with the entries that were removed by the
baseline. This suggests that asymmetry and covariance can have a regularising
effect in VPGs that is not present in FPGs --- a manifestation of the
properties of control variates~\cite{mohamed:2020:monte}.
Now, if we do not assume that the target functions are bounded, then the
linear term in \autoref{eq:var_diff} can grow arbitrarily in either
direction, but we typically require that rewards are restricted to some compact
subset $\mathcal{R} \subset \Re$ to avoid this. Below, we show that if a
similar requirement holds for each target function --- namely, that
$\inf_{\mathcal{S}, \mathcal{A}} \psi_j$ is well defined for each
$\psi_j\in\Psi$ --- then we can always construct a set of mappings that
constrain~\eqref{eq:var_diff} to be non-negative without biasing the gradient.

\begin{corollary}[Non-Negative Variance Reduction]\label{corr:non_negative}
    Let $\psi\!\left(s, \bm{a}\right)$ be of the form in
    \autoref{eq:linear_psi}.  If $\psi_j\!\left(s, \bm{a}\right) \geq
    \underline{\psi_j}$ for all $(s, \bm{a}) \in \mathcal{S}\times\mathcal{A}$
    and $j\in[m]$, with $\lvert\underline{\psi_j}\rvert < \infty$,
    then there exists a linear translation, $\psi_i \to \psi_i - \sum_{j=1}^m
    \lambda_j \underline{\psi_j}$, which leaves the gradient unbiased but
    yields $\Delta\mathbb{V}_i \geq 0$.\footnote{This inequality can be made strict if either $\alpha_i>0$ or $\beta_i>0$ --- where the former equates to having a non-zero trace of the Fisher information matrix --- and a small $\varepsilon > 0$ is added to the translation.}
\end{corollary}

\begin{figure*}[t]
    \centering
    \begin{subfigure}[b]{0.48\linewidth}
        \centering
        \includegraphics[width=\linewidth]{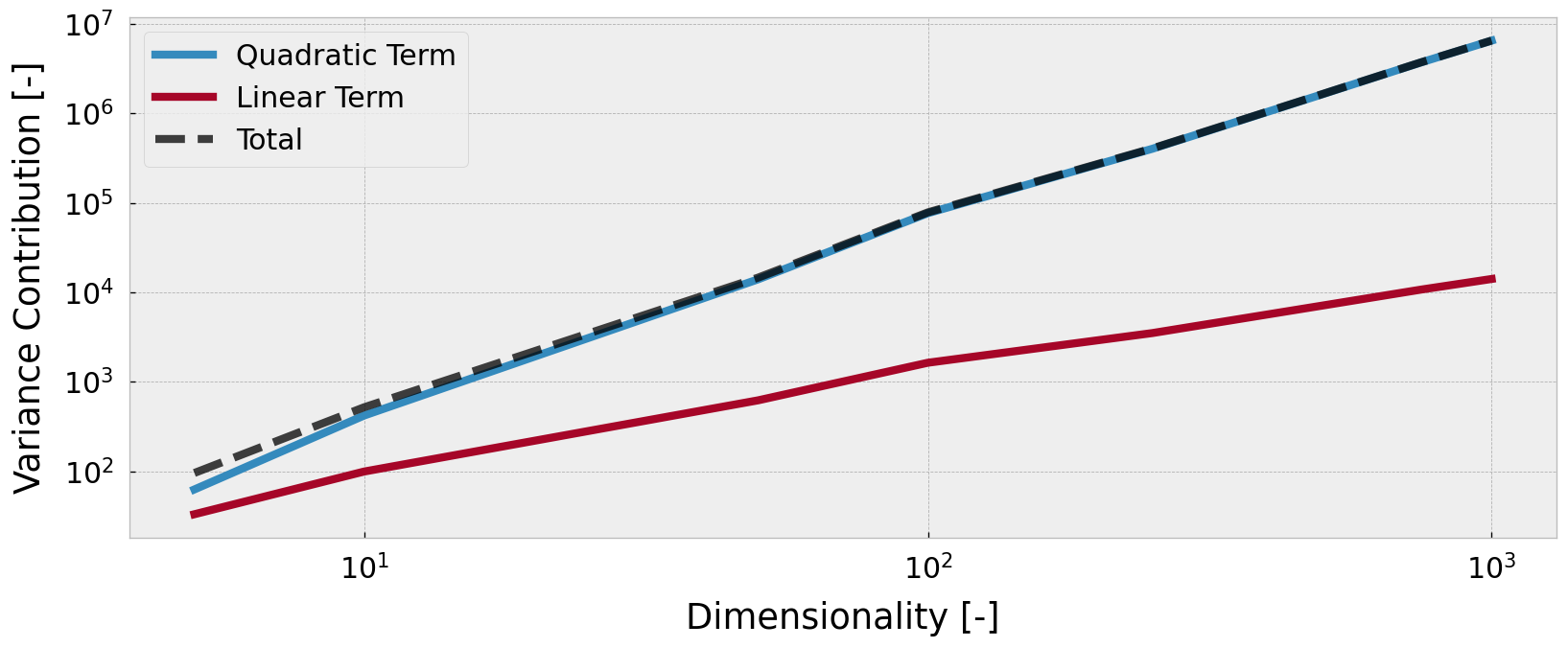}

        \caption{\textbf{Search bandit}: $\textrm{Cost}\!\left(\bm{a}\right) \doteq -\norm{\bm{a} -
        \bm{c}}_1$ with fixed centroid vector $\bm{c} \in \Re^n$.}
    \end{subfigure}
    \hfill
    \begin{subfigure}[b]{0.48\linewidth}
        \centering
        \includegraphics[width=\linewidth]{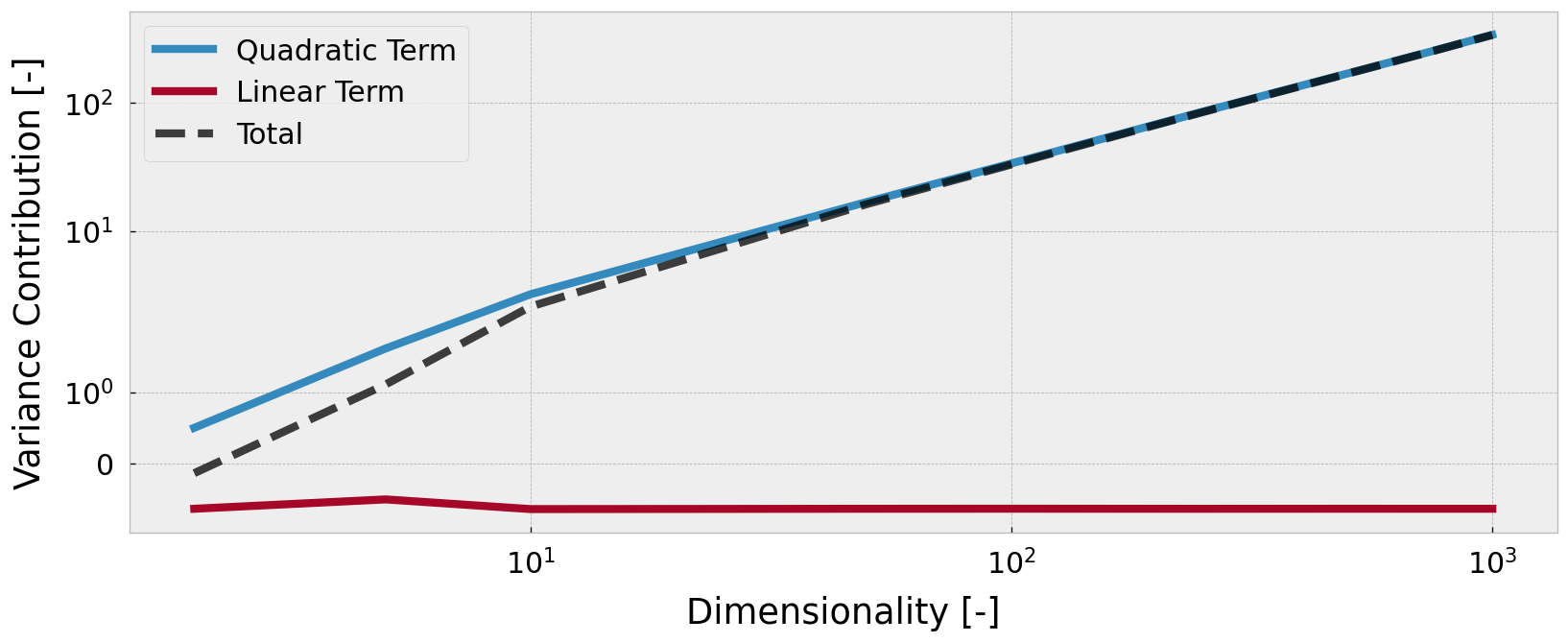}

        \caption{\textbf{ReLU bandit}: $\textrm{Cost}\!\left(\bm{a}\right) \doteq -\sum_{i=1}^n
        \max{\left(e_i a_i, 0\right)}$ with fixed sign vector $\bm{e} \in \{-1, 1\}^n$.}
    \end{subfigure}

    \caption{Variance decomposition on a symmetric log scale for two bandit
    problems as a function of action-space dimensionality. Each term was
    computed using Monte-Carlo estimation with $10^5$ samples and taking the
    arithmetic mean across all policy factors.}\label{fig:variance_decomp}
\end{figure*}

Interestingly, numerical experiments on a pair of continuum armed bandits
suggest that this transformation is seldom necessary; see
\autoref{fig:variance_decomp}. As the number of policy factors and targets
grow, so too does the potential discrepancy in magnitude between the quadratic
and linear terms in \autoref{eq:var_diff}. The former starts to
dominate even for small $\left\lvert\Sigma\right\rvert$. This is particularly
prevalent when the influence matrix $\bm{K}_\Sigma$ is very sparse and the
baselines have wide coverage over $\Psi$. In other words, applying FBs when the influence network is very dense or even complete will not
yield tangible benefits (e.g.\ in Atari games), but applying them to a problem with a rich structure, such as traffic networks, will almost certainly yield a significant net reduction in variance.

\paragraph{Bias-Variance Trade-Off.}
It is important to note that, in real-world problems, one does not always know
the exact structure of the influence network underling an MDP ex ante. This
poses a challenge since incorrectly \emph{removing} edges can introduce bias
and thus constrain the space of solutions that can be found by FPGs. Note,
however, that this may not always be a problem, since a small amount bias for a
large reduction in variance can be desirable. Furthermore, one could leverage
curriculum learning to train the policy on (presumed) influence networks with
increasing connectedness over time. This trade-off between bias and variance is
present in many machine learning settings, and depends strongly on the problem
at hand; we explore this empirically in \autoref{sec:numerical:traffic}.

\subsection{Minimum Factorisation}\label{sec:minimum_factorisation}
For many classes of fully-observable MDPs, any policy factorisation is
theoretically viable: we can fully factor the policy such that each action
dimension is independent of all others; or, at the other extreme, treat the
policy as a full joint distribution over $\mathcal{A}$. This holds because, in many
classes of (fully-observable) MDPs, there exists at least one deterministic
optimal policy~\cite{wiering:2012:reinforcement,puterman:2014:markov}. The
covariance acts as a driver of exploration, and it's initial value only affects
the rate of convergence.\footnote{Note that this is not true in general: the policy's covariance structure impacts the set of reachable solutions in
partially-observable MDPs and stochastic games, for example.} As a result, most
research uses an isotropic Gaussian with diagonal covariance to avoid the cost
of matrix inversion. This poses an interesting question: is there an
``optimal'' policy factorisation, $\Sigma_\mathcal{G}^\star$, associated with
an influence network $\mathcal{G}$? Below we offer a possible characterisation.

\begin{definition}[Minimum Factorisation]\label{def:minimum_factorisation}
    A minimum factorisation (MF), $\Sigma_\mathcal{G}^\star$, of an influence
    network, $\mathcal{G}$, is the \emph{minimum biclique vertex
    cover, disjoint amongst $I_\mathcal{A}$}.
\end{definition}

It follows from \autoref{def:minimum_factorisation} that for any
$\Sigma_\mathcal{G}^\star$, each $\sigma_i \in \Sigma_\mathcal{G}^\star$ is a
biclique (i.e.\ complete bipartite subgraph) of the original influence network
$\mathcal{G}$, and that the bipartite dimension is equal to the number of
policy factors. For example, one can trivially verify that
\autoref{fig:policy_factorisation:factored} is an MF of the original graph; see
also the reductions in \autoref{fig:graph:prototypes}. In essence, an MF describes a complete partitioning over action vertices --- so as to define a proper distribution --- where each group is a biclique with the same set of outgoing edges. The ``minimum'' qualifier then ensures that the maximum number of nodes are included in each of these groups, a property which allows us to prove the following result:

\begin{theorem}\label{thm:mfs}
    The MF $\Sigma_\mathcal{G}^\star$ always exists and is unique.
\end{theorem}

Minimum factorisation is a natural construction for the problem domains
studied in this paper; see \autoref{sec:numerical}. It also yields factored
policies which, generally, expose the minimum infimum bound on variance for a
given influence network. This follows from the fact that an MF yields the
greatest freedom to express covariance structure within each of the policy
factors whilst also maximising the quadratic term in \autoref{eq:var_diff}. In
fact, when each action corresponds to a single unique target, the MF enjoys a
lower bound on variance that is linear in the number of factors.  Finally, we
remark that, whilst closely related to vertex covering problems (which are
known to be \textsf{NP}-complete~\cite{karp:1972:reducibility}), we observed
experimentally that finding the MF can be done trivially in polynomial
time; see e.g.~\cite{fleischner:2009:covering}.

\section{Numerical Experiments}\label{sec:numerical}
\subsection{Search Bandits}\label{sec:numerical:bandit}
Consider an $n\doteq 1000$ dimensional continuum armed bandit with action space in $\Re^n$, and cost
function: $\text{Cost}\!\left(\bm{a}\right) \doteq \norm{\bm{a} - \bm{c}}_1 +
\lambda\,\zeta\!\left(\bm{a}\right)$, where $\bm{c} \in \Re^n$, $\lambda \geq
0$ and $\zeta : \mathcal{A} \to \Re_+$ is a penalty function.  This describes a
search problem in which the agent must locate the centroid $\bm{c}$ subject to
an action-regularisation penalty. It abstracts away the prediction aspects of
MDP settings, and allows us to focus only on scalability; note that this
problem is closely related to the bandit studied by \citet{wu:2018:variance}
for the same purpose. In our experiments, the centroids were initialised with a
uniform distribution, $\bm{c} \sim \mathcal{U}\left(-5, 5\right)$ and were held
fixed between episodes. The policy was defined as an isotropic Gaussian with
fixed covariance, $\textrm{diag}\!\left(\bm{1}\right)$, and initial location
vector $\bm{\mu} \doteq \bm{0}$. The influence network was specified such that
each policy factor, $\pi_{i, \bth}$ for $i\in[n]$, used a reward target
$\psi_i\!\left(\bm{a}\right) \doteq -\Delta_i\!\left(a_i\right) -
\lambda\zeta\!\left(a_i\right)$, with $\Delta_i\!\left(a\right) \doteq
\left\lvert a - c_i\right\rvert$, amounting to a collection of $n$ forks
(\autoref{fig:graph:prototypes:fork}). The parameter vector, $\bm{\mu}$, was
updated at each time step, and the hyperparameters are provided in the appendix.

\begin{figure*}
    \centering
    \begin{subfigure}[t]{0.48\linewidth}
        \centering
        \includegraphics[width=\linewidth]{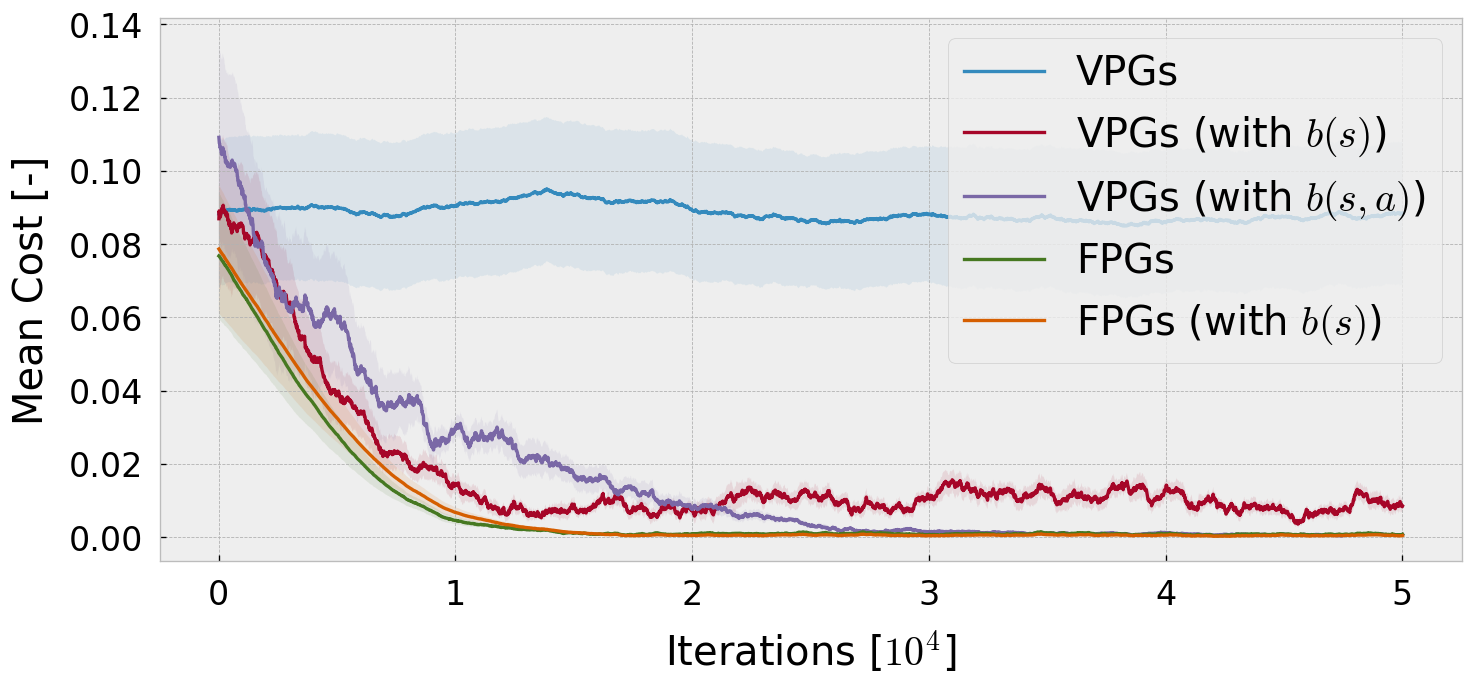}

        \caption{Convergence of FPGs (with and w/o additional state-dependent baseline) compared with VPGs using different baselines. The error bands denote the standard error on the mean over 10 random seeds.}\label{fig:search:baselines}
    \end{subfigure}
    \hfill
    \begin{subfigure}[t]{0.48\linewidth}
        \centering
        \includegraphics[width=\linewidth]{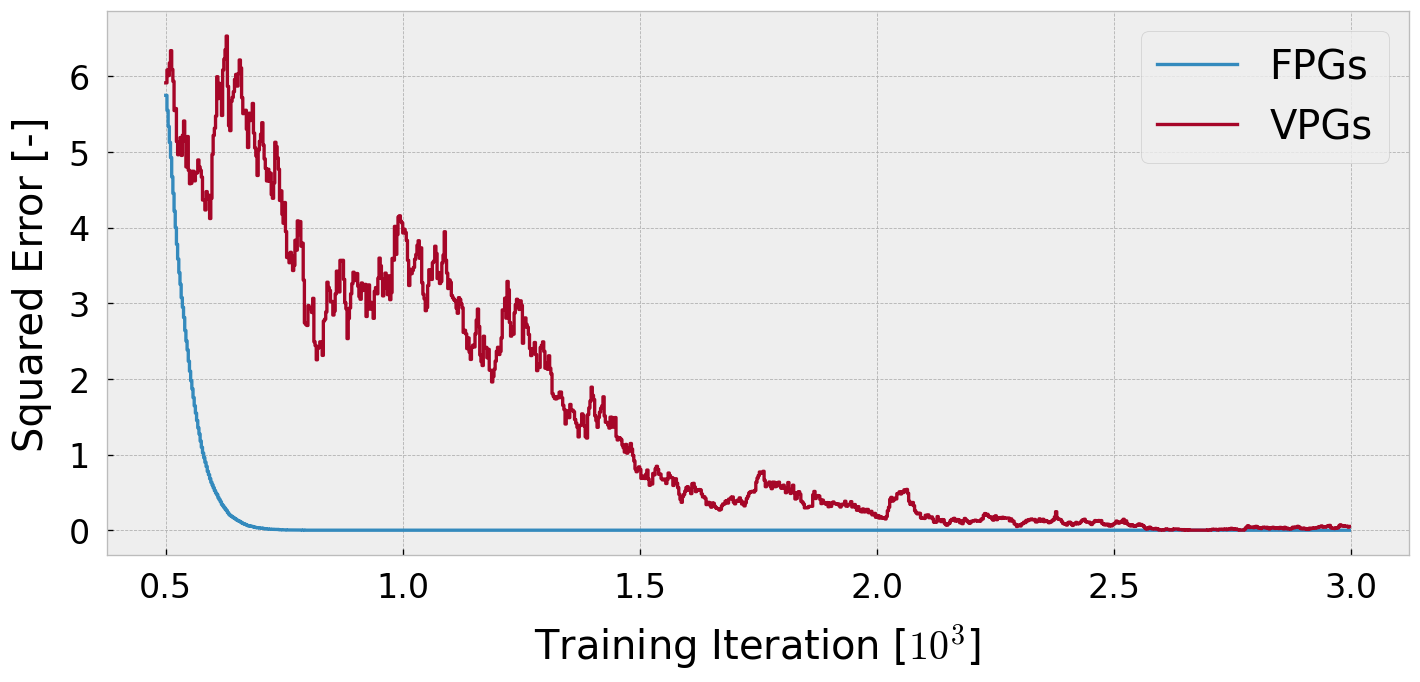}

        \caption{Squared error between $a_n$ and $c_n$ during learning for the search bandit with $k = n-1$.}\label{fig:search:aliasing}
    \end{subfigure}

    \caption{Performance analysis of FPGs on the search bandit domain.}\label{fig:search}
\end{figure*}

We began by examining the case where $\lambda = 0$ and the co-ordinate axes
were fully decoupled. For VPGs, we note that stability was only possible without a baseline
if an extremely low learning rate was used; see the appendix. Including a baseline dramatically
improved performance, with the action-dependent case, $b(s,a)$, also leading to
better asymptotic behaviour at the expense of a two orders of magnitude longer
train-time according to the wall-clock compared with all other algorithms (VPGs
and FPGs); see \autoref{tab:compute}. In comparison, FPGs, both with and without a
learnt state-dependent baseline, yielded significantly reduced variance,
leading to faster learning, more consistent convergence and highly robust
asymptotic stability.

\begin{table}
    \centering
    \caption{Empirical wall-clock estimates for the time-complexity (iterations per second) of VPGs and FPGs, with and without additional baselines. For each algorithm, the mean and sample standard deviation were computed across the 10 random seeds used to generate \autoref{fig:search:baselines}.}\label{tab:compute}

    \begin{tabular}{ll|ll}
        \toprule\toprule
        Method & Baseline & Mean [it / s] & Std Dev [it / s] \\
        \midrule
        \multirow{3}{*}{VPGs} & - & 10534 & 87 \\
        & $b(s)$ & 9885 & 81 \\
        & $b(s, a)$ & 80 & 1 \\
        \midrule
        \multirow{2}{*}{FPGs} & - & 9950 & 157 \\
        & $b(s)$ & 9670 & 126 \\
        \bottomrule\bottomrule
    \end{tabular}
\end{table}

We then studied the impact of coupling terms in the cost function; i.e.\
$\lambda > 0$. For this, we considered a family of penalties taking the form of
partially applied $\ell_2$ norms: $\zeta_k\!\left(\bm{a}\right) \doteq
\sqrt{\sum_{i=1}^k a_i^2}$, with $1 \leq k \leq n$. This set of functions
allowed us to vary the penalty attribution across the $n$ factors of
$\mathcal{A}$. Further examples
demonstrating the performance advantage of FPGs --- for $k = n$ and $k = n/2$ --- are given in the appendix. In
both cases, the improvement due to FB adjustments was found to be
non-negative for every combination of learning rate and action space. This
confirms that FPGs can indeed handle coupled targets and retains the variance
reduction benefits that were explored in \autoref{sec:variance_analysis}. As an
illustrative example, consider the case where $k=n-1$ and all but the last
action dimension are subject to a penalty. This is a particularly challenging
setting for VPGs because the magnitude of the combined cost function is much
greater than $\Delta_n\!\left(a_n\right)$, leading to an \emph{aliasing} of the
final component of the action vector in the gradient.  The result, as
exemplified in \autoref{fig:search:aliasing}, was that VPGs favoured reduction
of overall error, and was therefore exposed to poor per-dimension performance;
hence the increased noise in the $a_n$ error process. FPGs avoid this effect by
attributing gradients directly.

\subsection{Traffic Networks}\label{sec:numerical:traffic}
We now consider a classic traffic optimisation domain to demonstrate the
scalability of our approach to large real-world problems. In particular, we
consider variants of the ($3\times3$) grid network benchmark environment --- as
originally proposed by \citet{vinitsky:2018:benchmarks} --- that is provided by
the outstanding \texttt{Flow} framework~\cite{wu:2017:flow,liang:2018:rllib}.
In this setting, the RL agent is challenged with managing a set of traffic
lights with the objective of minimising the delay to vehicles travelling in the
network; the configuration should be taken as identical unless explicitly
stated. This requires a significant level of co-ordination, and indeed
multi-agent approaches have shown exemplary performance in this
space~\cite{vinitsky:2018:benchmarks,wei:2019:mixed}. However, much as with the
search bandit, the probability of aliasing effects increases substantially with
the number of lights/intersections; i.e.\ the dimensionality of the
action-space. This affects both single- and multi-agent approaches when the
global reward is used to optimise the policy.

\begin{figure*}
    \centering
    \begin{subfigure}[b]{0.48\linewidth}
        \centering
        \includegraphics[width=\linewidth]{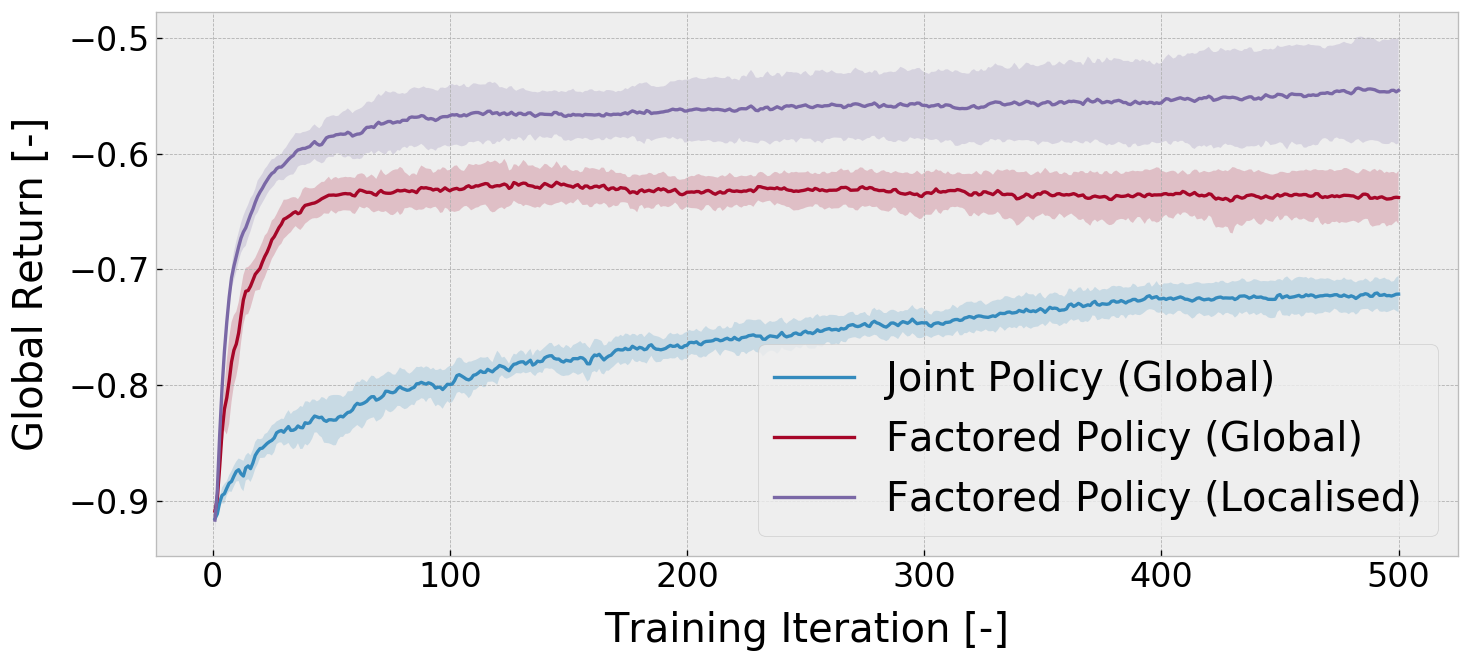}

        \caption{Convergence of joint/factored policies for the global reward,
        and using FPGs with the spatial baseline. Each curve depicts the
    mean value across 5 random seeds with standard error
    bands.}\label{fig:traffic:3b3}
    \end{subfigure}
    \hfill
    \begin{subfigure}[b]{0.48\linewidth}
        \centering
        \includegraphics[width=\linewidth]{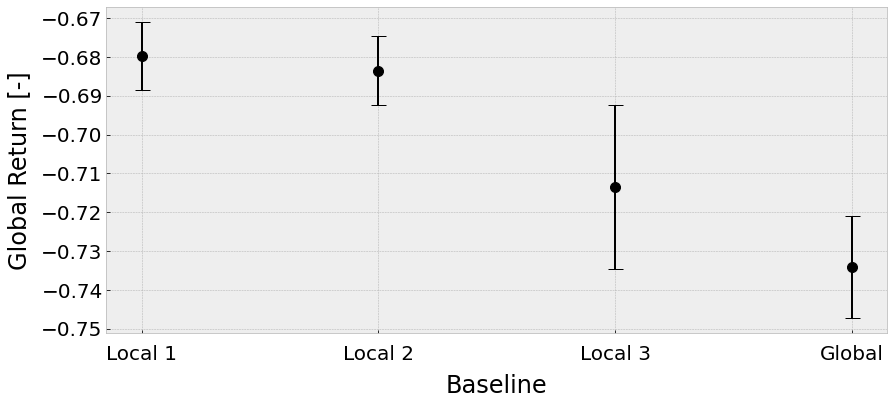}

        \caption{Performance degradation as a function of the $n$-level
        spatial baseline approximation in a $2\times 6$ grid network.
    Each point is the average terminal value across 5 seeds and with standard
    error bands.}\label{fig:traffic:2b6}
    \end{subfigure}

    \caption{Performance analysis of FPGs (with PPO and GAE) on the traffic network domain.}\label{fig:traffic}
\end{figure*}

To this end, we propose a ``baseline'' that \emph{removes reward terms
derived from streets/edges that are not directly connected to a given traffic
light}. This is based on the hypothesis that the local problem is sufficiently
isolated from the rest of the system that we may still find a (near)-optimal
solution; much as with local-form models~\cite{oliehoek:2012:influence}. Of
course, this could introduce bias at the cost of variance if we are incorrect
(see \autoref{sec:variance_analysis}), but this turns out to be an effective
trade-off as exemplified in \autoref{fig:traffic:3b3}.\footnote{Note that the standard errors may slightly underestimate the population level due to the low sample size.} In this plot we compare
the performance of three policies learnt using
PPO~\cite{schulman:2017:proximal} and GAE~\cite{schulman:2015:high} (with an
additional state-dependent baseline):
\begin{enumerate*}[label=(\arabic*)]
    \item a na\"{i}ve joint policy over the 9-dimensional action-space trained against the global reward;
    \item a shared policy trained on the global reward; and
    \item a shared policy using the local spatial baseline.
\end{enumerate*}
In methods 2 and 3, a shared policy refers to the use of a single univariate policy across all nine traffic lights, where only local information and identity variables are provided.
The global reward in this case was defined as the negative of the mean delay
introduced in the system minus a scaled penalty on vehicles at near standstill;
see the appendix for more details.

As expected, we observe that the
FBs improve learning efficiency, but, perhaps surprisingly, we
also find that the \emph{asymptotic behaviour is also superior}. We posit that
this relates to the fact that, with a fixed learning rate, stochastic gradient
descent cannot distinguish between points within a ball of the optimum solution
with radius that scales with the variance on the estimator. In other words,
significant reductions in variance, even if they introduce a small amount of
bias, may increase the likelihood of reaching the true optimal solution by
virtue of having much greater precision.

To better understand this trade-off, we also explored the impact of
``expanding'' the local baseline in a larger system of $2\times6$
intersections. With this new baseline we retain reward terms derived from
lights up to $n$ edges away in either the east or west directions. The variable
$n$ thus provides a dial to directly tweak the bias and variance of the policy
gradient estimator (i.e.\ increasing $n$ reduces bias but increases variance).
The result, as shown in \autoref{fig:traffic:2b6}, suggest that performance
decreases monotonically as a function of $n$. This corroborates the claim in
\autoref{sec:variance_analysis} that introducing some bias in exchange for a reduction in variance can be a worthwhile trade-off in large problems.

\section{Conclusion}
Factored policy gradients derive from the observation that many MOMDPs exhibit
redundancy in their reward structure. Here, we have characterised this
phenomenon using graphical modelling, and demonstrated that conditional independence
between factors of the action-space and the optimisation targets can be
exploited. The resulting family of algorithms
subsume many existing approaches in the literature. Our results in large-scale
bandit and concurrent traffic management problems suggest that FPGs are highly
suited to real-world problems, and may provide a way of scaling RL to domains
that have hitherto remained intractable. What's more, FPGs are compatible with
other techniques that improve policy gradient performance. For example, they
can be extended to use natural gradients by pre-multiplying \autoref{eq:fpg} by
the inverse Fisher information matrix~\cite{kakade:2001:natural}, and can even
use additional baselines to reduce variance even further, as in
\autoref{sec:numerical:traffic}. In future work we intend to address the
following interesting questions:
\begin{enumerate*}[label=(\alph*)]
    \item Can we infer/adapt the structure of influence networks?
    \item Are there canonical structures within $\bm{S}$ and $\bm{K}$?
    \item What theoretical insights can be derived from a more detailed
        analysis of the variance properties of FPGs?
\end{enumerate*}
We argue that factored approaches such as FPGs --- which are complementary to
ideas like influence-based abstraction~\cite{oliehoek:2012:influence} --- are
a promising direction for practical RL. Addressing some of these
questions, we believe, would thus be of great value to the community.



\begin{ack}
The authors would like to acknowledge our colleagues Joshua Lockhart, Jason
Long and Rui Silva for their input and suggestions at various key stages of the
research. This work was conducted by JPMorgan's AI Research group which has no external funding sources; i.e. it was a self-funded project. No sources of financial competing interests (or otherwise) are attributed to this line of research. 

\paragraph{Disclaimer}
This paper was prepared for informational purposes by the Artificial
Intelligence Research group of JPMorgan Chase \& Co and its affiliates (``J.P.\
Morgan''), and is not a product of the Research Department of J.P.\ Morgan.
J.P.\ Morgan makes no representation and warranty whatsoever and disclaims all
liability, for the completeness, accuracy or reliability of the information
contained herein. This document is not intended as investment research or
investment advice, or a recommendation, offer or solicitation for the purchase
or sale of any security, financial instrument, financial product or service, or
to be used in any way for evaluating the merits of participating in any
transaction, and shall not constitute a solicitation under any jurisdiction or
to any person, if such solicitation under such jurisdiction or to such person
would be unlawful.

\copyright{} 2021 JPMorgan Chase \& Co. All rights reserved.
\end{ack}

\bibliographystyle{plainnat}
\bibliography{main}

\clearpage
\appendix

\section{Factor Baselines}\label{app:baselines}
As shown in the main text, under the assumption that the influence network is unbiased, our factor baselines are indeed valid control variates. We prove this result below, repeating the statement itself for posterity and providing a supplementary lemma on control variates as a restatement of known results.

\begin{lemma}[\textbf{Control Variate}]\label{lem:control_variate}
    Let $X$, $Y$ and $Z$ be random variables where the law of $X$ conditional
    on $Z$ is denoted $\mathbb{P}_{\bth}\!\left(X \middle\vert Z\right)$, and
    $Y$ is independent of $X$ conditioned on $Z$; i.e.\ $X~\indep~Y~\vert~Z$.
    Then, we have that $\E{Y\,\gradbth\ln\mathbb{P}_{\bth}\!\left(X\right)} =
    0$.
\end{lemma}
\begin{proof}
    The proof follows from the law of iterated expectations:
    \begin{displaymath}
        \E{Y\,\gradbth\ln\mathbb{P}_{\bth}\!\left(X\right)} =
        \E{\CE{Y\,\gradbth\ln\mathbb{P}_{\bth}\!\left(X\right)}{Z}} =
        \E{\CE{Y}{Z}\CE{\gradbth\ln\mathbb{P}_{\bth}\!\left(X\right)}{Z}} = 0,
    \end{displaymath}
    since $\CE{\gradbth\ln\mathbb{P}_{\bth}\!\left(X\right)}{Z} = 0$.
\end{proof}

\begin{replemma}{lem:factor_baselines}
    Factor baselines are valid control variates if $\mathcal{G}_\Sigma$ is true to the MDP (i.e.\ unbiased).
\end{replemma}
\begin{proof}
    Consider an objective $J\!\left(\bth\right)$ of the form defined in
    \autoref{eq:linear_psi}, a factored influence network $\mathcal{G}_\Sigma$ and
    a $\Sigma$-factored policy $\cpt_{\bth}(\bm{a}|s) \doteq \prod_{i=1}^n
    \cpt_{i,\bth}(\sigma^\pi_i\!\left(\bm{a}\right)|s)$. Now, let us define a
    stochastic policy gradient estimator
    \begin{displaymath}
        \gradbth J\!\left(\bth\right) =
        \EE{\pi_{\bth},\rho_{\pi_{\bth}}}{\bm{g}\!\left(s, \bm{a}\right) \doteq
        \sum_{i=1}^n \left[\psi\!\left(s, \bm{a}\right) + b^\textsc{C}_i\!\left(s,
        \bar\sigma^\pi_i\!\left(\bm{a}\right)\right)\right] \bm{z}_i},
    \end{displaymath}
    where $\bm{z}_i \doteq
    \gradbth\ln\cpt_{i,\bth}(\sigma^\pi_i\!\left(\bm{a}\right)|s)$ and
    $b^\textsc{C}_i\!\left(s, \bar\sigma^\pi_i\!\left(\bm{a}\right)\right)$ is the
    $i$\textsuperscript{th} factor baseline (see \autoref{def:factor_baselines}).
    If $\mathcal{G}_\Sigma$ is unbiased then we have mutual independence between each action partition and, since $b^\textsc{C}_i\!\left(s,
    \bar\sigma^\pi_i\!\left(\bm{a}\right)\right)$ depends only on $s$ and
    $\bar\sigma^\pi_i\!\left(\bm{a}\right)$ --- i.e.\ the action elements that are
    not in the support of $\pi_{i,\bth}$ --- we can readily apply
    \autoref{lem:control_variate}, thus concluding the proof.
\end{proof}

\subsection{Optimality}
In contrast to the factor baselines, solving for the optimal baseline in general is a non-trivial challenge.
Indeed, the results presented by \citet{wu:2018:variance} rely on a key
assumption that the policy factors do not share parameters in order to simplify
the analysis; i.e.\ that $\dot<\bm{z}_i, \bm{z}_j> \approx 0$ for any $i, j \in
\left[\left\lvert\Sigma\right\rvert\right]$. Below we explain why this is a
difficult problem, and leave it to future work to find the solution.

For notational convenience, let $\bm{g}\!\left(s, \bm{a}\right) \doteq
\sum_{i=1}^n \bm{g}_i\!\left(s, \bm{a}\right)$ such that the total variance on
the gradient is given by
\begin{equation}\label{eq:pgrad:total_var}
    \V{\bm{g}\!\left(s, \bm{a}\right)} = \sum_{i=1}^n \sum_{j=1}^n
    \textrm{Cov}\!\left[\bm{g}_i\!\left(s, \bm{a}\right), \bm{g}_j\!\left(s,
    \bm{a}\right)\right].
\end{equation}
The $n$ \emph{optimal baselines} are given by the values that minimise
\autoref{eq:pgrad:total_var}; i.e. $b_i^\star\!\left(s,
\bar\sigma^\pi_i\left(\bm{a}\right)\right) \doteq \argmin_{b_i}
\V{\bm{g}\!\left(s, \bm{a}\right)}$ for all $i \in [n]$.  To solve this
problem, we first apply the factor baseline decomposition such that
$b_i^\star\!\left(s, \bar\sigma^\pi_i\!\left(\bm{a}\right)\right) =
b^\textsc{V}_i\!\left(s, \bar\sigma^\pi_i\!\left(\bm{a}\right)\right) +
b^\textsc{C}_i\!\left(s, \bar\sigma^\pi_i\!\left(\bm{a}\right)\right)$.  This
implies that the optimisation problem can be reduced to finding
$\argmin_{b^\textsc{V}_i} \V{\bm{g}\!\left(s, \bm{a}\right)}$ when $b_i$ is
replaced with $b_i^\star$ for all $i \in [n]$. Now, let $\bm{x}_i \doteq
\left[\bm{K}_\Sigma\,\bm{\psi}\!\left(s, \bm{a}\right)\right]_i \bm{z}_i$ and
$\bm{y}_i \doteq b_i^\textsc{V}\!\left(s,
\bar\sigma^\pi_i\!\left(\bm{a}\right)\right) \bm{z}_i$ such that
$\bm{g}_i\!\left(s, \bm{a}\right) = \bm{x}_i + \bm{y}_i$. Note that while
$\bm{y}_i$ depends on the full action, $\bm{x}_i$ depends only on the actions
influencing the targets in $\left[\bm{K}_\Sigma\,\bm{\psi}\!\left(s,
\bm{a}\right)\right]_i$. Removing terms that are independent of
$b_i^\textsc{V}$ thus yields the following:
\begin{align*}
    \argmin_{b^\textsc{V}_i} \V{\bm{g}}
        &= \argmin_{b^\textsc{V}_i} \left\{\V{\bm{g}_i} +
        \sum_{j\ne i}^n\textrm{Cov}\!\left[\bm{g}_i, \bm{g}_j\right]\right\},
        \\
        &= \argmin_{b^\textsc{V}_i} \left\{\V{\bm{x}_i} +
        \V{\bm{y}_i} + 2\,\textrm{Cov}\!\left[\bm{x}_i, \bm{y}_i\right] +
        \sum_{j\ne i}^n\textrm{Cov}\!\left[\bm{x}_i, \bm{x}_j\right] +
        \textrm{Cov}\!\left[\bm{x}_i, \bm{y}_j\right] +
        \textrm{Cov}\!\left[\bm{y}_i, \bm{x}_j\right] +
        \textrm{Cov}\!\left[\bm{y}_i, \bm{y}_j\right]\right\}, \\
        &= \argmin_{b^\textsc{V}_i} \left\{\V{\bm{y}_i} +
        2\,\textrm{Cov}\!\left[\bm{x}_i, \bm{y}_i\right] +
        \sum_{j\ne i}^n\textrm{Cov}\!\left[\bm{x}_i, \bm{y}_j\right] +
        \textrm{Cov}\!\left[\bm{y}_i, \bm{x}_j\right] +
        \textrm{Cov}\!\left[\bm{y}_i, \bm{y}_j\right]\right\}.
\end{align*}
To solve the equation above, we first expand each component and remove any
redundant terms. For the variance on $\bm{y}_i$, we have that
\begin{align}
    \V{\bm{y}_i}
        &= \EE{\bm{a}}{\left(b_i^\textsc{V}\right)^2 \dot<\bm{z}_i, \bm{z}_i>}
        + \dot<\EE{\bm{a}}{b_i^\textsc{V} \bm{z}_i}, \EE{\bm{a}}{b_i^\textsc{V}
        \bm{z}_i}>, \nonumber \\
        &= \EE{\bm{a}}{\left(b_i^\textsc{V}\right)^2 \dot<\bm{z}_i, \bm{z}_i>}
        + \dot<\EE{\bar\sigma_i^\pi\!\left(\bm{a}\right)}{b_i^\textsc{V}}
        \EE{\sigma_i^\pi\!\left(\bm{a}\right)}{\bm{z}_i},
        \EE{\bar\sigma_i^\pi\!\left(\bm{a}\right)}{b_i^\textsc{V}},
        \EE{\sigma_i^\pi\!\left(\bm{a}\right)}{\bm{z}_i}>, \nonumber \\
        &= \EE{\bm{a}}{\left(b_i^\textsc{V}\right)^2 \dot<\bm{z}_i, \bm{z}_i>},
        \nonumber \\
        &= \EE{\sigma_i^\pi\!\left(\bm{a}\right)}{ \dot<\bm{z}_i, \bm{z}_i>}
        \EE{\bar\sigma_i^\pi\!\left(\bm{a}\right)}{\left(b_i^\textsc{V}\right)^2}.
        \label{eq:variance_yi}
\end{align}
It follows from this analysis that the covariance between $\bm{y}_i$ and
$\bm{y}_j$ for any $i, j \in \left[\left\lvert\Sigma\right\rvert\right]$, with
$i \ne j$, is given by
\begin{equation}\label{eq:covariance_yi_yj}
    \textrm{Cov}\!\left[\bm{y}_i, \bm{y}_j\right] = \EE{\bm{a}}{b_i^\textsc{V}
    b_j^\textsc{V} \dot<\bm{z}_i, \bm{z}_j>}.
\end{equation}
Finally, we can expand the covariance between $\bm{x}_i$ and $\bm{y}_i$,
\begin{align}
    \textrm{Cov}\!\left[\bm{x}_i, \bm{y}_i\right]
        &= \EE{\bm{a}}{\left[\bm{K}_\Sigma\,\bm{\psi}\right]_i b_i^\textsc{V}
        \dot<\bm{z}_i, \bm{z}_i>} +
        \dot<\EE{\bm{a}}{\left[\bm{K}_\Sigma\,\bm{\psi}\right]_i \bm{z}_i},
        \EE{\bm{a}}{b_i^\textsc{V} \bm{z}_i}>, \nonumber \\
        &= \EE{\bm{a}}{\left[\bm{K}_\Sigma\,\bm{\psi}\right]_i b_i^\textsc{V}
        \dot<\bm{z}_i, \bm{z}_i>} +
        \dot<\EE{\bm{a}}{\left[\bm{K}_\Sigma\,\bm{\psi}\right]_i \bm{z}_i},
        \EE{\bar\sigma_i^\pi\!\left(\bm{a}\right)}{b_i^\textsc{V}}
        \EE{\sigma_i^\pi\!\left(\bm{a}\right)}{\bm{z}_i}>, \nonumber \\
        &= \EE{\bm{a}}{\left[\bm{K}_\Sigma\,\bm{\psi} \right]_i b_i^\textsc{V}
        \dot<\bm{z}_i, \bm{z}_i>}, \nonumber \\
        &= \EE{\bar\sigma_i^\pi\!\left(\bm{a}\right)}{b_i^\textsc{V}}
        \EE{\sigma_i^\pi\!\left(\bm{a}\right)}{\left[\bm{K}_\Sigma\,\bm{\psi}
        \right]_i \dot<\bm{z}_i, \bm{z}_i>}, \label{eq:covariance_xi_yi}
\end{align}
and similarly resolve the cross-covariance terms:
\begin{equation}\label{eq:covariance_xi_yj}
    \textrm{Cov}\!\left[\bm{x}_i, \bm{y}_j\right] =
    \dot<\EE{\bar\sigma_i^\pi\!\left(\bm{a}\right)}{b_i^\textsc{V}\bm{z}_i},
    \EE{\sigma_i^\pi\!\left(\bm{a}\right)}{\left[\bm{K}_\Sigma\,\bm{\psi}
    \right]_i \bm{z}_j}>.
\end{equation}

The quantities above provide us with a platform to find solutions. For example,
the optimal baseline approximation proposed by \citet{wu:2018:variance} can be
found if we assume that $\dot<\bm{z}_i, \bm{z}_j> \approx 0$ since
\autoref{eq:covariance_yi_yj} and \autoref{eq:covariance_xi_yi} go to zero.
However, in the general case the problem is not quite so simple. The reason for
this is that the baselines interact via the cross-covariance term in
\autoref{eq:covariance_yi_yj}. As a result, we cannot solve for each
$b_i^\textsc{V}$ independently of the others. Instead, we have a system of
polynomial equations which may not have a unique solution. In fact, since each
equation has degree $d = 2$, it follows the number of solutions can be as large
$d^{\left\lvert\Sigma\right\rvert}$. In general, there are very few methods
that can solve these type of systems, and those that can are limited to bounds
of approximately $d^{\left\lvert\Sigma\right\rvert} \approx 20$. It seems
reasonable to assume that any solution, while viable, would be computationally
impractical, but we leave it to future work to establish this result formally.

\section{Factored Policy Gradients}
The validity of factor baselines, as shown in the previous section, extends to policy gradient themselves. As discussed in the main text, we can show that FPGs are unbiased and satisfy certain variance bounds compared with conventional policy gradients. We restate the original propositions below and provide the proofs in full.

\begin{repproposition}{prop:fpg}
    Take a $\Sigma$-factored policy $\pt_\bth(\bm{a}|s)$ and
    $\left\lvert\bm{\theta}\right\rvert \times \left\lvert\Sigma\right\rvert$
    matrix of scores $\bm{S}\!\left(s, \bm{a}\right)$. Then, for target vector
    $\bm{\psi}\!\left(s, \bm{a}\right)$ and multipliers $\bm{\lambda}$, the
    \emph{FPG} estimator
    \begin{displaymath}
        \bm{g}^\textsc{C}\!\left(s, \bm{a}\right) \doteq \bm{S}\!\left(s,
        \bm{a}\right)\bm{K}_\Sigma\,\bm{\lambda}\circ\bm{\psi}\!\left(s,
        \bm{a}\right),
    \end{displaymath}
    is an unbiased estimator of the true policy gradient; i.e.\ $\gradbth J\!\left(\bth\right) = \EE{\pi_{\bth},
    \rho_{\pi_{\bth}}}{\bm{g}^\textsc{C}\!\left(s, \bm{a}\right)}$.
\end{repproposition}
\begin{proof}
    Let $\mathcal{G}_\Sigma$ denote an $\Sigma$-factored influence network with
    policy $\cpt_{\bth}(\bm{a}|s) \doteq \prod_{i=1}^n
    \cpt_{i,\bth}(\sigma^\pi_i\!\left(\bm{a}\right)|s)$, and global target function
    $\psi\!\left(s, \bm{a}\right) = \sum_{j=1}^m \lambda_j \psi_j\!\left(s,
    \sigma_j\!\left(\bm{a}\right)\right) = \dot<\bm{\lambda}, \bm{\psi}\!\left(s,
    \bm{a}\right)>$. The score matrix, $\bm{S}\!\left(s, \bm{a}\right) \doteq
    \left[\bm{z}_i^\top, \dots, \bm{z}_n^\top\right]^\top$, then has size
    $\left\lvert\bm{\theta}\right\rvert \times n$, where $\bm{z}_i \doteq
    \gradbth\ln\cpt_{i,\bth}(\sigma^\pi_i\!\left(\bm{a}\right)|s)$. From this we
    can express the conventional policy gradient with no baseline as the linear
    product $\bm{g}\!\left(s, \bm{a}\right) = \bm{S}\!\left(s, \bm{a}\right)
    \bm{J}_{n,m} \bm{\psi}\!\left(s, \bm{a}\right)$, where $\bm{J}_{n,m}$ is the
    $n\times m$ all-ones matrix. By \autoref{lem:factor_baselines} the factor
    baselines, $\left[\left(1 - \bm{K}_\Sigma\right) \bm{\psi}\!\left(s,
    \bm{a}\right)\right]_i$, are valid control variates and thus have expected
    value of zero under $\pi$. This means that they can be subtracted without
    introducing bias in the policy gradient, yielding
    \begin{displaymath}
        \bm{g}\!\left(s, \bm{a}\right) = \underbrace{\bm{S}\!\left(s,
        \bm{a}\right) \bm{J}_{n,m} \bm{\psi}\!\left(s,
        \bm{a}\right)}_{\textrm{Vanilla PG}} - \underbrace{\bm{S}\!\left(s,
        \bm{a}\right) \left(1 - \bm{K}_\Sigma\right) \bm{\psi}\!\left(s,
        \bm{a}\right)}_{\textrm{Factor Correction}} = \bm{S}\!\left(s,
        \bm{a}\right) \bm{K}_\Sigma\, \bm{\psi}\!\left(s, \bm{a}\right).
    \end{displaymath}
    It follows that $\gradbth J\!\left(\bth\right) = \EE{\pi_{\bth},
    \rho_{\pi_{\bth}}}{\bm{g}^\textsc{C}\!\left(s, \bm{a}\right)}$ since
    $\EE{\pi_{\bth}, \rho_{\pi_{\bth}}}{\bm{g}^\textsc{V}\!\left(s, \bm{a}\right)}
    = \EE{\pi_{\bth}, \rho_{\pi_{\bth}}}{\bm{g}^\textsc{C}\!\left(s,
    \bm{a}\right)}$ which concludes the proof.
\end{proof}

\begin{repproposition}{prop:variance_reduction}
    Let $\bm{g}_i$ denote a gradient estimate for the $i$\textsuperscript{th}
    factor of a $\Sigma$-factored policy $\pi_{\bth}$
    (\autoref{eq:policy_factorisation}). Then, $\Delta\mathbb{V}_i \doteq
    \V{\bm{g}_i^\textsc{V}} - \V{\bm{g}_i^\textsc{C}}$, satisfies
    \begin{displaymath}
        \Delta\mathbb{V}_i = \alpha_i
        \,\EE{\bar\sigma_i^\pi\!\left(\bm{a}\right)}{\left(b_i^\textsc{C}\right)^2}
        + 2\beta_i\EE{\bar\sigma_i^\pi\!\left(\bm{a}\right)}{b_i^\textsc{C}},
    \end{displaymath}
    where $\bm{z}_i \doteq \gradbth\ln\cpt_{i,\bth}(\bm{a}|s)$, $\alpha_i
    \doteq \EE{\sigma_i^\pi\!\left(\bm{a}\right)}{\dot<\bm{z}_i, \bm{z}_i>}
    \geq 0$ and $\beta_i \doteq
    \EE{\sigma_i^\pi\!\left(\bm{a}\right)}{\dot<\bm{z}_i, \bm{z}_i>\left(\psi +
    b_i^\textsc{C}\right)}$.
\end{repproposition}

\begin{proof}
    First, let us denote by $\bm{X}$ and $\bm{Y}$ two (possibly dependent) random
    variables, with $\bm{Z} \doteq \bm{X} - \bm{Y}$ such that
    \begin{align*}
        \Delta\mathbb{V} \doteq \V{\bm{X}} - \V{\bm{Y}}
            &= \V{\bm{Z} + \bm{Y}} - \V{\bm{Y}}, \\
            &= \V{\bm{Z}} + \V{\bm{Y}} + 2\textrm{Cov}\!\left[\bm{Y},
            \bm{Z}\right] - \V{\bm{Y}}, \\
            &= \V{\bm{Z}} + 2\,\textrm{Cov}\!\left[\bm{Y}, \bm{Z}\right].
    \end{align*}
    From \autoref{prop:fpg}, we can express the vanilla and factored policy gradient
    estimators for the $i$\textsuperscript{th} factor as $\psi\,\bm{S}_{\cdot,i}$
    and $\left(\psi - b_i^\textsc{C}\right) \bm{S}_{\cdot,i}$, respectively, where
    the function arguments have been omitted for clarity. Assigning these values to
    $\bm{X}$ and $\bm{Y}$ we arrive at the equality relations
    \begin{align*}
        \V{\bm{Z}} = \V{b_i^\textsc{C}\bm{z}_i}
            &= \EE{\pi}{\dot<\bm{z}_i, \bm{z}_i> \left(b_i^\textsc{C}\right)^2} \\
        \textrm{Cov}\!\left[\bm{Y}, \bm{Z}\right] = \textrm{Cov}\!\left[\left(\psi
        - b_i^\textsc{C}\right) \bm{z}_i, b_i^\textsc{C}\bm{z}_i\right]
            &= \EE{\pi}{\dot<\bm{z}_i, \bm{z}_i> \left(\psi -
            b_i^\textsc{C}\right)b_i^\textsc{C}}.
    \end{align*}
    The former follows from the fact that $\EE{\pi}{\bm{S}_{\cdot, i}} = 0$ for all
    $i$, and latter by noting that $\E{\bm{Z}} = 0$ due to
    \autoref{lem:factor_baselines}. We can now exploit the independencies implied
    by the influence network, $\mathcal{G}_\Sigma$, to give
    \begin{displaymath}
        \Delta\mathbb{V}_i = \EE{\sigma_i^\pi\!\left(\bm{a}\right)}{\dot<\bm{z}_i,
            \bm{z}_i>}
            \EE{\bar\sigma_i^\pi\!\left(\bm{a}\right)}{\left(b_i^\textsc{C}\right)^2}
            + 2\EE{\sigma_i^\pi\!\left(\bm{a}\right)}{\dot<\bm{z}_i, \bm{z}_i>
                \left(\psi -
            b_i^\textsc{C}\right)}\EE{\bar\sigma_i^\pi\!\left(\bm{a}\right)}{b_i^\textsc{C}},
    \end{displaymath}
    This is the desired result and thus concludes the proof.
\end{proof}

\begin{repcorollary}{corr:non_negative}
    Let $\psi\!\left(s, \bm{a}\right)$ be of the form in
    \autoref{eq:linear_psi}.  If $\psi_j\!\left(s, \bm{a}\right) \geq
    \underline{\psi_j}$ for all $(s, \bm{a}) \in \mathcal{S}\times\mathcal{A}$
    and $j\in[m]$, with $\left\lvert\underline{\psi_j}\right\rvert < \infty$,
    then there exists a linear translation, $\psi_i \to \psi_i - \sum_{j=1}^m
    \lambda_j \underline{\psi_j}$, which leaves the gradient unbiased but
    yields $\Delta\mathbb{V}_i \geq 0$.
\end{repcorollary}

\begin{proof}
    Take a target set $\Psi$ and let $\underline{\psi_j} \doteq \inf_{\mathcal{S},
    \mathcal{A}} \psi_j$ for each $\psi \in \Psi$. The unbiasedness claim follows
    from the fact that these terms go to zero in expectation when weighted by the
    score functions; they are constants. The variance claim is also trivial, since
    $\psi_j + \sum_{k=1}^m \lambda_k \inf_{\mathcal{S}, \mathcal{A}} \psi_k$ are
    non-negative and, due to the summation over all $k\in[m]$, no CB can yield a
    negative value. Each term in \autoref{eq:var_diff}
    (\autoref{prop:variance_reduction}) must also be non-negative, which concludes the
    proof.
\end{proof}

\section{Minimum Factorisation}\label{app:mf}
The minimum factorisation of an influence network provides a natural way of partitioning action nodes into independent policy distributions. In the main text it was also stated that such a characterisation is natural to the problems we study. We repeat this result below and provide the proof herein.

\begin{reptheorem}{thm:mfs}
    The MF $\Sigma_\mathcal{G}^\star$ always exists and is unique.
\end{reptheorem}
\begin{proof}
    Bipartite graphs always have at least one valid biclique and
    thus MF. Now, for uniqueness, let $\mathcal{G}$ denote an influence network. If
    $\mathcal{G}$ is complete, then we automatically satisfy the uniqueness
    property since the MF will contain a single biclique that covers all vertices
    in $I_\mathcal{A}$. If $\mathcal{G}$ is incomplete, then the proof can be shown
    through contradiction. Suppose that $A$ and $B$ are both MFs and therefore
    correspond to minimum biclique vertex covers, disjoint amongst $I_\mathcal{A}$.
    We know then that $A$ and $B$ must have the same dimensionality since they are
    optimal --- i.e.\ contain the same number of bicliques --- but, if they are
    distinct, then there must also exist at least one biclique $a \in A$ that is
    not in $B$. Since both MFs are defined over the same graph $\mathcal{G}$, the
    elements of $a$ must be distributed between at least 2 distinct bicliques in
    $B$. However, if this is the case, the union of these subgraphs would also form
    a valid biclique. The new cover, $B'$, containing the merged bicliques is valid
    and has dimensionality $\left\lvert B'\right\rvert < \left\lvert B\right\rvert
    = \left\lvert A\right\rvert$. This implies that neither $A$ nor $B$ can be MFs.
    Since the same must be true for any $A$ and $B$, it follows that there can be
    only one MF, thus concluding the proof.
\end{proof}


\section{Search Bandit}\label{sec:appendix:search_bandit}
The search bandit was designed to exhibit an influence network as illustrated in \autoref{fig:graph:search_bandit}. Below we summarise the hyperparameters for the two key experiments --- namely the baseline comparison (BC) and aliasing demonstration (AD):
\begin{description}
    \item[BC] All algorithms were trained using a learning rate of 0.5 except for VPGs w/o a baseline which was only stable with a step size of 0.001. The state-based (i.e. scalar) baselines, $b(s) = b$, were trained using temporal-difference methods with a learning rate of 0.1. The action-dependent baseline, $b(s, a) \doteq -\left\lvert\left\lvert\bm{a} - \bm{w}\right\rvert\right\rvert_1 / \left\lvert\mathcal{A}\right\rvert$, was similarly trained using SARSA with a learning rate of 0.1.\footnote{In this formalism we only have sub-derivatives. For simplicity we simply assigned the gradient when a given action was equal to the weight.} An additional 1000 episodes were also used at the start of each run to pre-train the baseline if used.
    \item[AD] In the aliasing experiment, both VPGs and FPGs were trained for a 100-dimensional action-space with a regularisation penalty of $\lambda = 0.01$ on the first 99 action-components. VPGs were instantiated with a learning rate of 0.001, and FPGs with a rate of 0.01.
\end{description}

\paragraph{Additional results.}
In addition to the results presented in the paper, we also include Figures~\ref{fig:search:stability}-\ref{fig:search:discrepancy}. These explore the impact of the factor baseline across a set of dimensionalities and learning rates. We show that VPGs are very sensitive to the learning rate, especially when $\left\lvert\mathcal{A}\right\rvert$ is large. FPGs, on the other hand, converge on the optimal solution consistently regardless of the problem instance. Similarly, we show that the mean number of steps required to reach such a solution for a finite budget is much lower for FPGs compared with VPGs.

\paragraph{Implications for MDPs.}
The search bandit is an interesting problem environment because, in many ways,
it can emulate the learning process in arbitrary MDPs. This follows because,
without loss of generality, we can always transform an MDP into a (possibly
infinite) set of continuum multi-armed bandits, one occupying every unique
state $s \in \mathcal{S}$. The question is how to define the cost function in
order to achieve some form of equivalence. For example, if we consider
deterministic policies, then we can clearly define the cost to be
$\textrm{Cost}\!\left(\bm{a}\right) \doteq \left\lvert\left\lvert \bm{a} -
\pi^\star\!\left(s\right) \right\rvert\right\rvert_p$, for $p \geq 1$, and have
the same solution set as given under Bellman optimality. This implies that the performance observed in the search
bandit it likely to tell us about the performance in full MDPs. The results presented in \autoref{sec:numerical:bandit} may thus provide
evidence that FPGs will outperform VPGs for arbitrarily challenging
MDPs.

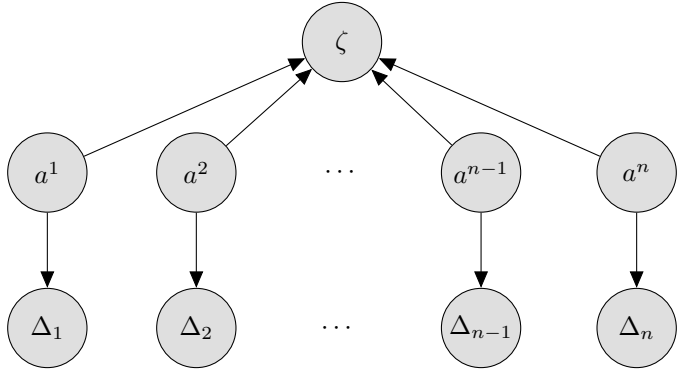
\begin{figure*}[!ht]
    \centering
    \begin{tikzpicture}
        \node (rdots) {$\cdots$};
        \node[obs, left=of psi2, minimum size=30pt] (psi1) {$\Delta_1$};
        \node[obs, left=of rdots, minimum size=30pt] (psi2) {$\Delta_2$};
        \node[obs, right=of rdots, minimum size=30pt] (psinm1) {$\Delta_{n-1}$};
        \node[obs, right=of psinm1, minimum size=30pt] (psin) {$\Delta_n$};

        \node[obs, above=of psi1, minimum size=30pt] (a1) {$a^{1}$};
        \node[obs, above=of psi2, minimum size=30pt] (a2) {$a^{2}$};
        \node[obs, above=of psinm1, minimum size=30pt] (anm1) {$a^{n-1}$};
        \node[obs, above=of psin, minimum size=30pt] (an) {$a^{n}$};
        \node (adots) at ($(a1)!0.5!(an)$) {$\cdots$};

        \node[obs, above=of adots, minimum size=30pt] (psi0) {$\zeta$};

        \edge {a1} {psi0,psi1};
        \edge {a2} {psi0,psi2};
        \edge {anm1} {psi0,psinm1};
        \edge {an} {psi0,psin};
    \end{tikzpicture}

    \caption{Influence network of the search bandit problem with optional
    coupling term.}
    \label{fig:graph:search_bandit}
\end{figure*}

\begin{figure*}[t]
    \centering
    \begin{subfigure}[b]{0.48\linewidth}
        \centering
        \includegraphics[width=\linewidth]{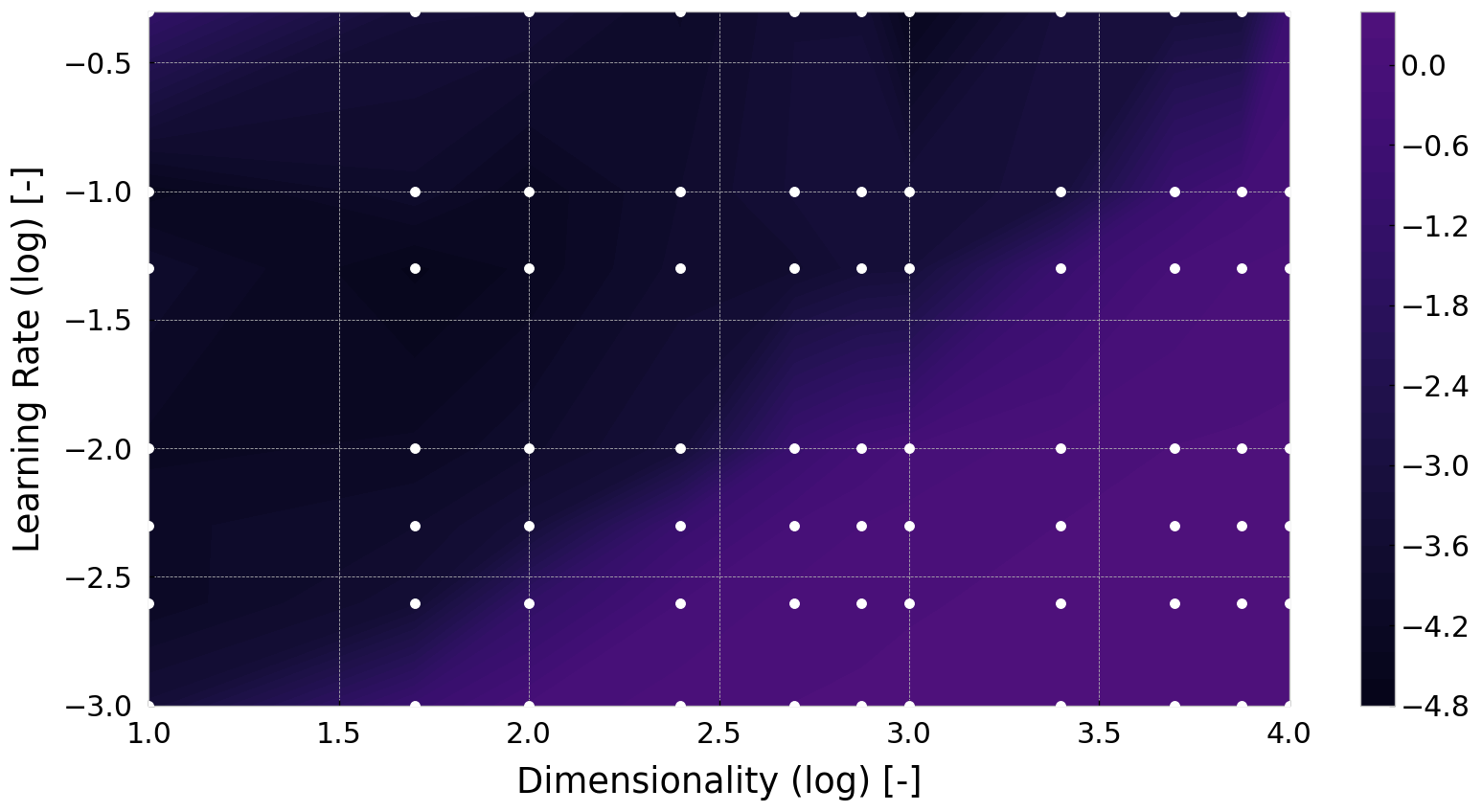}

        \caption{FPGs}
    \end{subfigure}
    \hfill
    \begin{subfigure}[b]{0.48\linewidth}
        \centering
        \includegraphics[width=\linewidth]{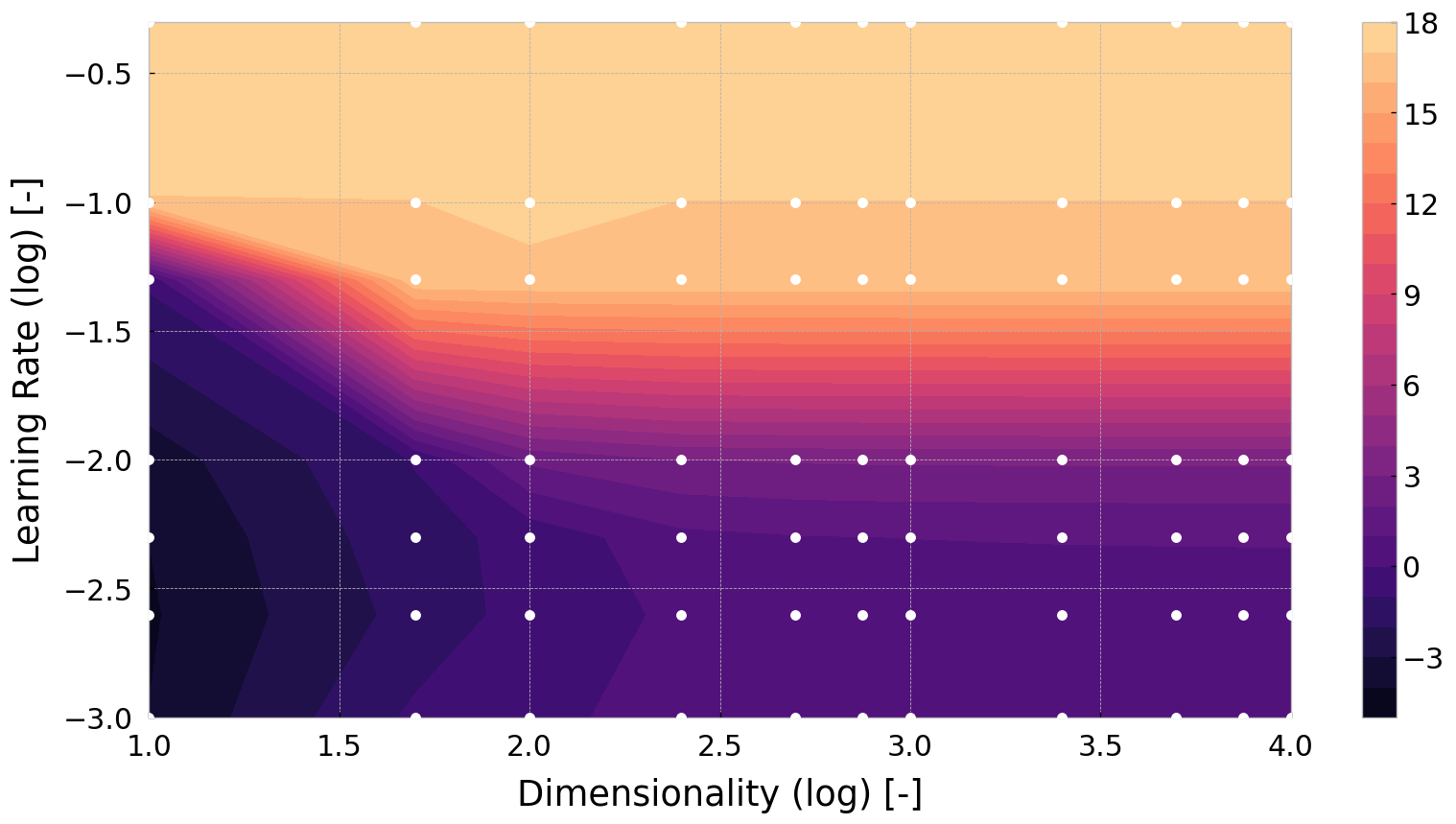}

        \caption{VPGs}
    \end{subfigure}

    \caption{Mean optimality gap after $2\times 10^5$ training iterations. The
    $z$-axis is given in a log scale and each point was computed from 16 random
    samples under the assumption of a Gamma distribution (optimality gap is
    lower bounded at zero).}\label{fig:search:stability}
\end{figure*}

\begin{figure*}[!ht]
    \centering
    \begin{subfigure}[b]{0.48\linewidth}
        \centering
        \includegraphics[width=\linewidth]{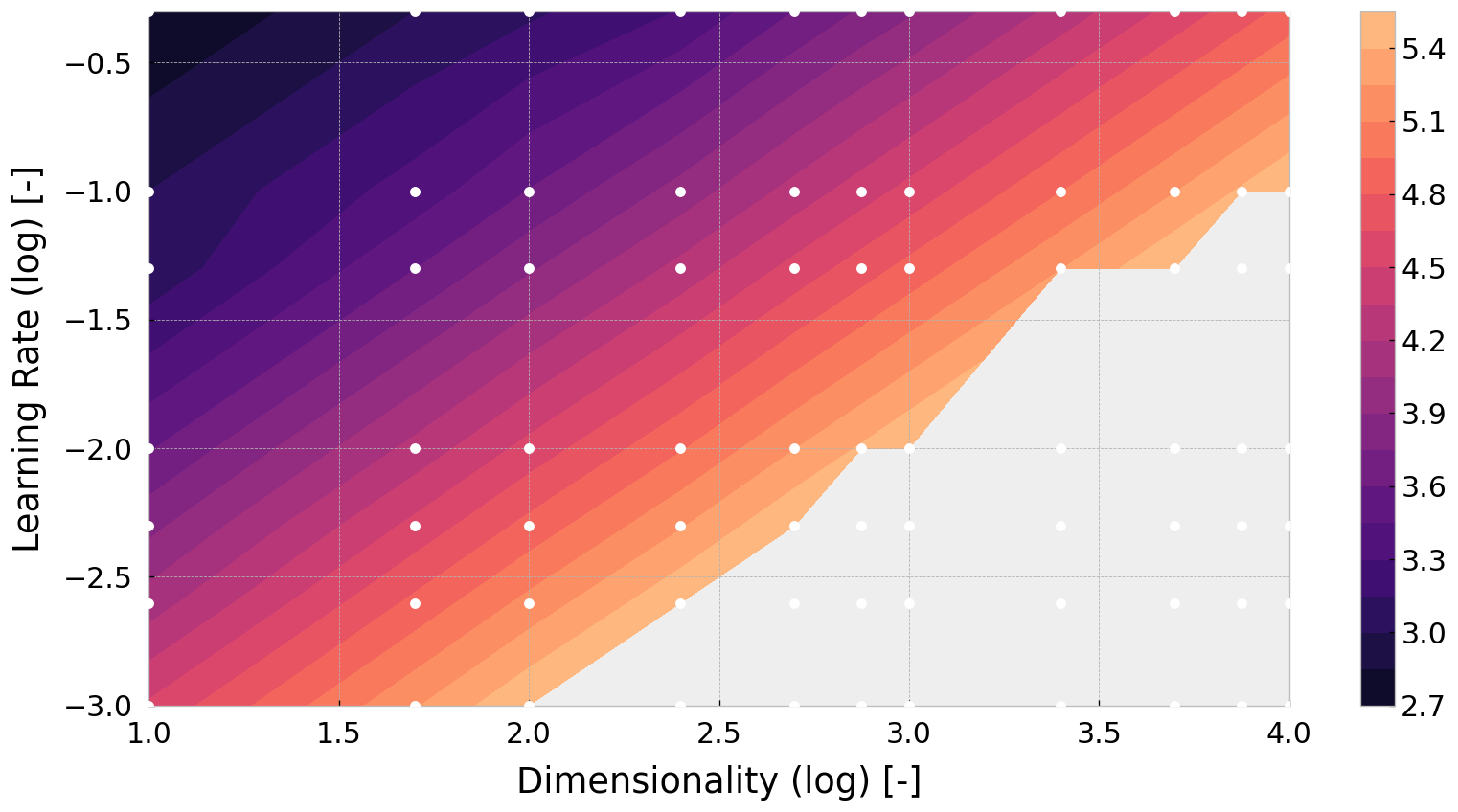}

        \caption{FPGs}
    \end{subfigure}
    \hfill
    \begin{subfigure}[b]{0.48\linewidth}
        \centering
        \includegraphics[width=\linewidth]{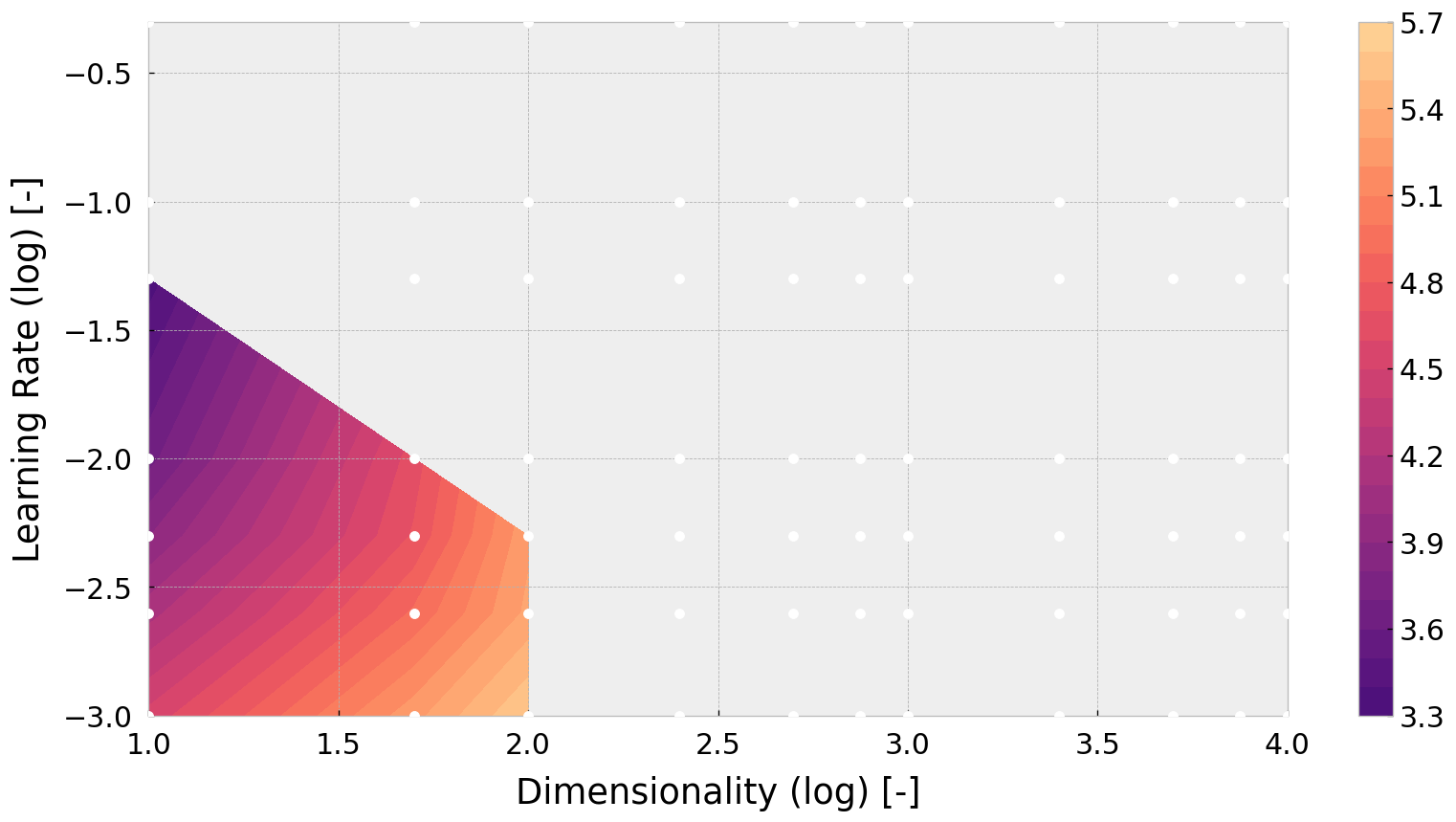}

        \caption{VPGs}
    \end{subfigure}

    \caption{Mean number of time steps required to reach an optimality gap of
        $0.1$, up to a limit of $5\times 10^5$ training iterations; see
        \autoref{fig:search:stability}. The $z$-axis is given in a log scale,
        and unfilled (grey) regions depict either divergence or a failure to
    terminate in the allotted time. Each point was computed from 16 random
    samples under the assumption of a Gamma distribution (time is lower bounded
    at zero).}\label{fig:search:rate}
\end{figure*}

\begin{figure*}[!ht]
    \centering
    \begin{subfigure}[b]{0.48\linewidth}
        \centering
        \includegraphics[width=\linewidth]{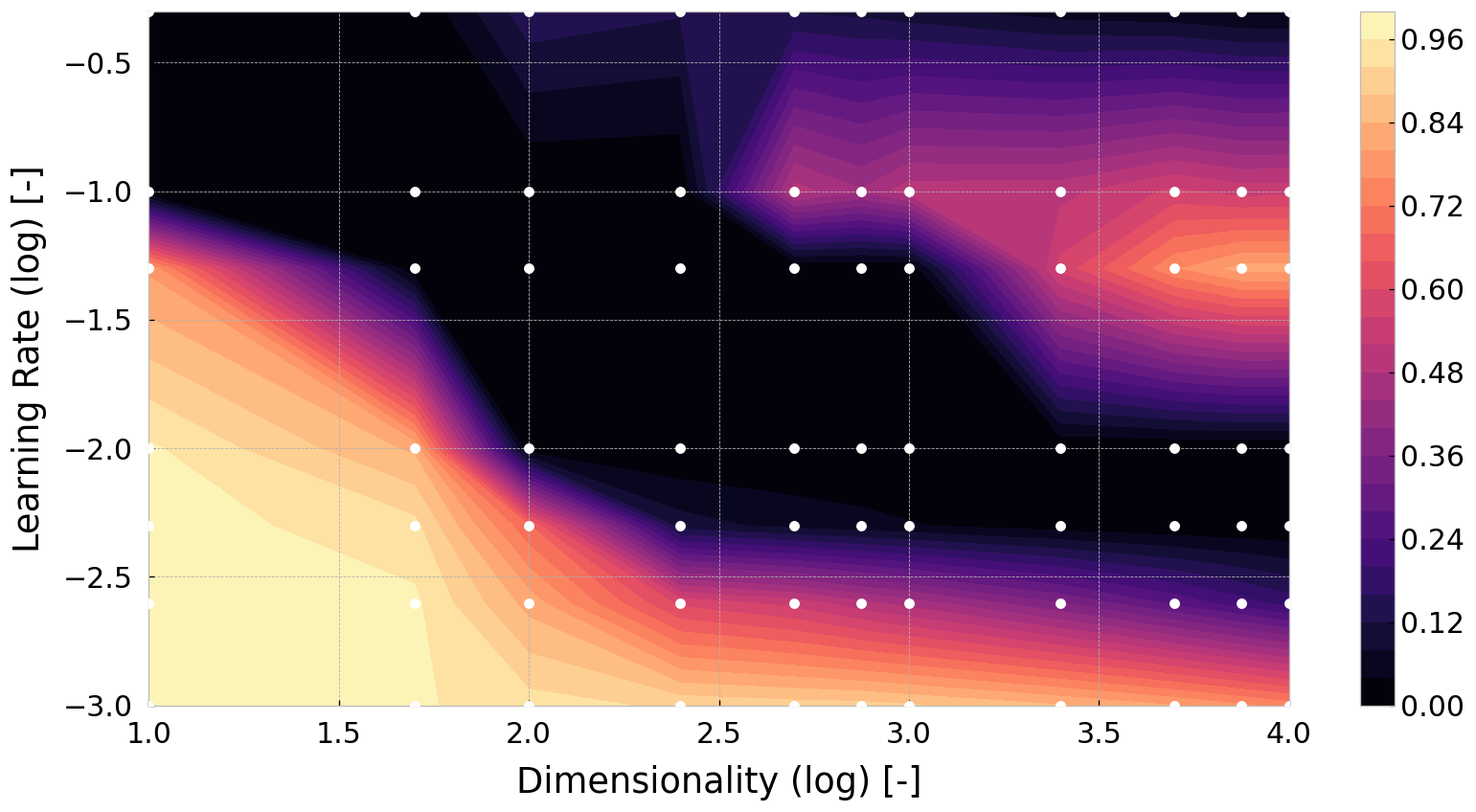}

        \caption{$\xi_n\!\left(\bm{a}\right) = \norm{\bm{a}}_2$ and $\lambda =
        0.01$}
    \end{subfigure}
    \hfill
    \begin{subfigure}[b]{0.48\linewidth}
        \centering
        \includegraphics[width=\linewidth]{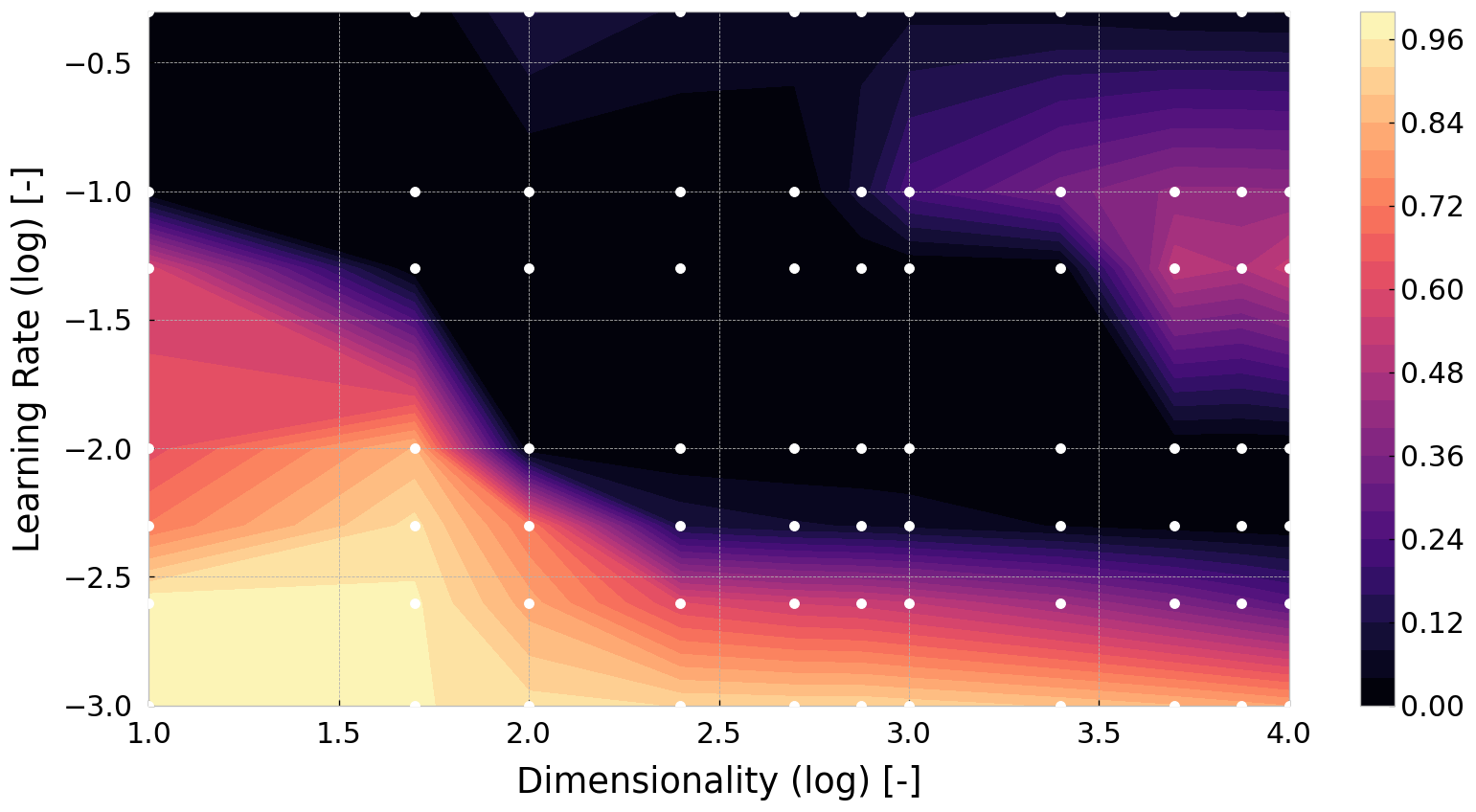}

        \caption{$\xi_{n/2}\!\left(\bm{a}\right) = \sqrt{\sum_{i=1}^{n/2}
        a^2_i}$ and $\lambda = 0.01$}
    \end{subfigure}

    \caption{Ratio between the mean optimality gaps for FPGs over VPGs after
        $2\times 10^5$ training iterations. Smaller values indicate that FPGs
        achieved a lower error relative to VPGs. Each point is the ratio of the
        two means, each computed using 16 random samples under the assumption
    of a Gamma distribution (optimality gaps are lower bounded at
    zero).}\label{fig:search:discrepancy}
\end{figure*}

\section{Traffic Systems}\label{app:traffic}
The traffic experiment were kept as close as possible to the benchmark specification for the grid problem provided by \textsf{Flow}~\cite{wu:2017:flow}. In particular, we based the code of the ``examples/exp\_configs/rl/multiagent/multiagent\_traffic\_light\_grid.py'' and ``examples/exp\_configs/rl/signleagent/singleagent\_traffic\_light\_grid.py'' files on commit ID 4e47f7a. The only changes that were made were to update the topology of the grid (i.e. $3 \times 3$ and $2 \times 6$), and to unify the reward function. We outline all the specific details below.

\paragraph{Reward functions.}
In order to unify the reward function across domains we implemented a custom variant of the ``mean delay'' case that worked for single- or multi-agent approaches. In particular, we changed the summation to only consider a subset of the edges in the network which allowed for localised computation. This can be done very easily in the Flow framework.

\paragraph{Traffic system parameters.}
The traffic intersection problem was instantiated with either a $3\times 3$ topology, or a $2\times 6$ topology, depending on the experiment. In all cases, an edge inflow of 300 was used, with initial speed of 30. The inner edges were given a length of 300, with the final edge in a route having length 100, and starting edge having length 300. Cars were created using the SimCarFollowingController, and SumoCarFollowingParams with a minimum gap of 2.5, maximum speed of 30, decelleration rate of 7.5 and ``right of way'' speed mode. The environment itself was initialised with target velocity of 50, switch time of 3, number of locally observed cars at 2, ``actuated'' TL type, and 4 locally observed edges.

\paragraph{Learning hyperparameters.}
In all cases we leveraged RLLIB's implementation of PPO with GAE~\cite{liang:2018:rllib} using discount factor of 0.999, a Monte-Carlo interpolation rate of $\lambda = 0.97$, KL-target of $0.02$, value function clipping bound at $10^4$, and learning rate of $5\times 10^{-4}$. The policy was parameterised using a three-layer neural network with 32 units at each of the three hidden layers. A total of 50 CPUs were used, each generating a single rollout at each iteration with a horizon of 400 steps.

\end{document}